\documentclass[journal]{IEEEtran}

\usepackage[noadjust]{cite}

\usepackage[utf8]{inputenc} 
\usepackage[T1]{fontenc}    
\usepackage{amsmath,amssymb,amsfonts,amsthm}
\usepackage{mathtools}
\usepackage{graphicx}
\usepackage{multirow}
\usepackage{textcomp}
\usepackage{xcolor}
\usepackage{tcolorbox}
\usepackage{tikz}
\tcbuselibrary{skins}
\usetikzlibrary{calc}
\usetikzlibrary{shapes.multipart}
\usepackage{mathrsfs} 
\usepackage[hidelinks]{hyperref} 
\usepackage{wrapfig} 
\usepackage{booktabs}
\usepackage{algorithm}
\usepackage[noend]{algpseudocode}

\makeatletter
\renewcommand{\ALG@beginalgorithmic}{\small}
\makeatother

\ifCLASSOPTIONcompsoc
\usepackage[caption=false,font=normalsize,labelfon
t=sf,textfont=sf]{subfig}
\else
\usepackage[caption=false,font=footnotesize]{subfig}
\fi

\usepackage{etoolbox}
\makeatletter
\patchcmd{\@makecaption}
  {\scshape}
  {}
  {}
  {}
\patchcmd{\@makecaption}
  {\\}
  {.\ }
  {}
  {}
\patchcmd{\@makecaption}
  {\centering}
  {}
  {}
  {}
\makeatother

\def\BibTeX{{\rm B\kern-.05em{\sc i\kern-.025em b}\kern-.08em
    T\kern-.1667em\lower.7ex\hbox{E}\kern-.125emX}}
    
\def\<#1>{\boldsymbol{\mathbf{#1}}} 
\providecommand{\norm}[1]{\lVert#1\rVert}
\providecommand{\abs}[1]{\lvert#1\rvert}
\providecommand{\relu}{\mathrm{relu}}

\newtheorem{dfn}{Definition}
\newtheorem{prp}{Proposition}
\newtheorem{thm}{Theorem}
\newtheorem{lem}{Lemma}

\author{Trevor~Avant~and~Kristi~A.~Morgansen
\thanks{Both authors are with the Department
of Aeronautics \& Astronautics, University of Washington, Seattle,
WA, 98115 USA e-mail: trevoravant@gmail.com, morgansn@uw.edu.}
\thanks{This work was supported by ONR grant N00014-17-1-2623.}
\thanks{Code is available at \url{https://github.com/uwaa-ndcl/local_lipschitz}}}

\title{Analytical bounds on the local Lipschitz constants of ReLU networks}

\begin{document}
\bstctlcite{myctl}

\maketitle

\begin{abstract}
In this paper, we determine analytical upper bounds on the local Lipschitz constants of feedforward neural networks with ReLU activation functions.
We do so by deriving Lipschitz constants and bounds for ReLU, affine-ReLU, and max pooling functions, and combining the results to determine a network-wide bound.
Our method uses several insights to obtain tight bounds, such as keeping track of the zero elements of each layer, and analyzing the composition of affine and ReLU functions.
Furthermore, we employ a careful computational approach which allows us to apply our method to large networks such as AlexNet and VGG-16.
We present several examples using different networks, which show how our local Lipschitz bounds are tighter than the global Lipschitz bounds.
We also show how our method can be applied to provide adversarial bounds for classification networks.
These results show that our method produces the largest known bounds on minimum adversarial perturbations for large networks such as AlexNet and VGG-16.
\end{abstract}

\section{Introduction} \label{sec:intro}




Although neural networks have proven to be very adept at handling image processing tasks, they are also often very sensitive. For many networks, a small perturbation of the input can produce a huge change in the output \cite{Szegedy}.
Due to neural networks' high-dimensionality and complex constitutive functions, sensitivity is difficult to analyze, and as a result, is still not theoretically well-understood.
Nevertheless, neural networks are currently being applied to a wide range of tasks, including safety-critical applications such as autonomous driving. In order to safely incorporate neural networks into the physical world, it is necessary to develop a better theoretical understanding of their sensitivity.

The high sensitivity of deep neural networks has been noted as early as \cite{Szegedy}. This work also conceived the idea of adversarial examples, which are small perturbations to an input that cause a network to misclassify (and have since become a popular area of research in their own right \cite{Goodfellow2014}).
Sensitivity can be characterized in a variety of ways, one of which being the input-output Jacobian \cite{Novak, Sokolic}. Although the Jacobian gives a local estimate of sensitivity, it generally cannot be used to provide any meaningful guarantees.

Another characterization of sensitivity is the Lipschitz constant, which describes how much the output of a function can change with respect to changes in the input. Lipschitz constants can be computed as a global measure which applies to any input, or as a local measure which applies only to a specific set of possible inputs. Furthermore, the Lipschitz constant can take several forms depending on which norm is used to define it (e.g., the 1-, 2-, or $\infty$-norm).

Regardless of whether the measure is global or local, or which norm is used to define it, analytically computing the exact Lipschitz constant is challenging due to the complexity and high-dimensionality of neural networks. Although this task was approached in \cite{Jordan} using mixed integer programming, the resulting method can only be applied to very small networks, and only works with respect to the 1- and $\infty$-norms.

As exact computation of the Lipschitz constant is formidable, the next best option is to determine an upper bound. A standard upper bound on the global Lipschitz constant was described in \cite{Szegedy}, and is computed by taking the product of the Lipschitz bounds of each function in a network. This bound is simple and can be computed for larger networks, but it has the downside of being very conservative. 

Recently, several studies have explored using optimization-based approaches to bound or approximate the Lipschitz constant of neural networks. The work of \cite{Scaman} presents two algorithms to bound the Lipschitz constant: AutoLip and SeqLip. AutoLip reduces to the global Lipschitz bound, while SeqLip is an algorithm which requires a greedy approximation for larger networks. The work of \cite{Latorre} presents a sparse polynomial optimization method (LiPopt) to compute bounds on Lipschitz constants, but relies on the network being sparse which often requires the network to be pruned. A semidefinite programming technique (LipSDP) is presented in \cite{Fazlyab} to compute Lipschitz bounds, but in order to apply it to larger networks, a relaxation must be used which invalidates the guarantee. Another approach is that of \cite{Zou}, in which linear programming is used to estimate Lipschitz constants. The downside to all of these approaches is that they usually can only be applied to small networks, and also often have to be relaxed, which invalidates any guarantee on the bound. Also of note are several other works that have considered constraining Lipschitz constants as a means to regularize a network \cite{Gouk,Terjek,Bartlett}.

Network sensitivity is also often analyzed in regards to the robustness of classification networks against adversarial examples \cite{Tjeng, Peck, Tsuzuku, Weng2}. This area of research is often closely related to Lipschitz analysis, but since the focus is on the specific task of adversarial examples for classification networks, it is often unclear how or if these techniques can be adapted to provide Lipschitz bounds. However, we do note that the method we present in this paper involves a few key insights that have also been utilized in \cite{Weng2, Tjeng}, specifically, bounding the Lipschitz constants of the ReLU, and determining that certain output elements of a layer are always zero.


In summary, few techniques are available which can provide guaranteed Lipschitz bounds, and the ones that do only work for small networks. In this paper, we present a method which provides guaranteed local Lipschitz bounds which can be computed for large networks. Furthermore, as Lipschitz constants are directly related to adversarial bounds, we also show how our method can produce guaranteed bounds on the minimum magnitude of adversarial examples. Our method produces a bound that is orders of magnitude tighter than the bound derived from the global Lipschitz constant (the only other bound which can be applied to large networks), so our method represents a significant improvement in certifying adversarial bounds.

Our analysis focuses on two types of functions, affine-ReLU and max pooling, which serve as the building blocks of many networks such as AlexNet \cite{Krizhevsky} and the VGG networks \cite{Simonyan}.


The remainder of this paper is organized as follows. In Section \ref{sec:norms} we make a short note about which norm we consider in the paper. In Section \ref{sec:lipschitz} we discuss global and local Lipschitz constants. In Sections \ref{sec:relu_lipschitz}, \ref{sec:affine_relu_lipschitz}, and \ref{sec:max_pooling_lipschitz} we derive Lipschitz constants or bounds for ReLU, affine-ReLU, and max pooling functions, respectively. In Section \ref{sec:network_wide_bounds} we describe how to combine our bounds to calculate the local Lipschitz constant of an entire feedforward neural network. In Section \ref{sec:computational_techniques} we discuss computational techniques which make it possible to apply our results to large networks. In Section \ref{sec:simulations} we apply our method to various networks, and show how our Lipschitz bounds can be used to determine bounds on adversarial examples for classification networks. The paper concludes in Section \ref{sec:conclusion} with a summary and possible next steps.





\section{Note about norms} \label{sec:norms}

Note that in this paper, we let $\norm{\cdot}$ denote the 2-norm. However, many of our results, such as those in Section \ref{sec:lipschitz}, hold for any norm. Extending our results to other norms is an area of future work.

\section{Lipschitz constants} \label{sec:lipschitz}

\subsection{Global Lipschitz constants}

In this paper we will analyze sensitivity using Lipschitz constants, which measure how much the output of a function can change with respect to changes in the input.
\begin{dfn}
The \textbf{global Lipschitz constant} of a function $\<f>: \mathbb{R}^n \rightarrow \mathbb{R}^m$ is the minimal $L \geq 0$ such that
\begin{equation}
\norm{\<f>(\<x>_2) - \<f>(\<x>_1)} \leq L \norm{\<x>_2 - \<x>_1}, ~~~~ \forall \<x>_1,\<x>_2 \in \mathbb{R}^n .
\label{eq:global_lipschitz_constant_original}
\end{equation}
\end{dfn}
Note that some authors define any $L$ that satisfies the inequality above as ``a Lipschitz constant'', but we will define ``the Lipschitz constant'' as the minimal $L$ for which this inequality holds, and we refer to any larger value as an upper bound. We can solve for $L$ in \eqref{eq:global_lipschitz_constant_original} as
\begin{equation}
L = \sup_{\<x>_1 \neq \<x>_2} \frac{\norm{\<f>(\<x>_2) - \<f>(\<x>_1)}}{\norm{\<x>_2 - \<x>_1}} .
\label{eq:global_lipschitz_constant}
\end{equation}
Note that excluding points such that $\<x>_1{=}\<x>_2$ does not affect the supremization above since these points satisfy \eqref{eq:global_lipschitz_constant_original} for any $L$.

\subsection{Local Lipschitz constants}

While the global Lipschitz constant is computed with respect to all possible inputs in $\mathbb{R}^n$, we can also compute a local Lipschitz constant with respect to only a specific set of inputs. In this paper, we will define the local Lipschitz constant with respect to a nominal input $\<x>_0 \in \mathbb{R}^n$ and set of all possible inputs $\mathcal{X} \subset \mathbb{R}^n$. Note that we have used the symbol ``$L$'' to denote the global Lipschitz constant, and will overload our notation and use the symbol ``$L(\<x>_0, \mathcal{X})$'' to denote the local Lipschitz constant.
\begin{dfn}
The \textbf{local Lipschitz constant} of a function $\<f>: \mathbb{R}^n \rightarrow \mathbb{R}^m$, with respect to nominal input $\<x>_0 \in \mathbb{R}^n$ and set of all possible inputs $\mathcal{X} \subset \mathbb{R}^n$, is the minimal $L(\<x>_0, \mathcal{X})$ such that
\begin{equation}
\norm{\<f>(\<x>) - \<f>(\<x>_0)} \leq L(\<x>_0, \mathcal{X}) \norm{\<x> - \<x>_0} , ~~~~ \forall \<x> \in \mathcal{X} .
\label{eq:local_lipschitz_constant_original}
\end{equation}
\end{dfn}
As we did in \eqref{eq:global_lipschitz_constant}, we can solve for $L(\<x>_0, \mathcal{X})$ in \eqref{eq:local_lipschitz_constant_original} which yields
\begin{align}
L( \<x>_0, \mathcal{X} ) &\coloneqq \sup_{\substack{\<x> \in \mathcal{X} \\ \<x> \neq \<x>_0}} \frac{\norm{\<f>(\<x>) - \<f>(\<x>_0)}}{\norm{\<x> - \<x>_0}} .
\label{eq:local_lipschitz_constant}
\end{align}
As with the global Lipschitz constant, excluding points such that $\<x>{=}\<x>_0$ does not affect the result since these points satisfy \eqref{eq:local_lipschitz_constant_original} for any $L(\<x>_0,\mathcal{X})$. In this paper we will often leave the ``$\<x>{\neq}\<x>_0$'' out of the subscript \eqref{eq:local_lipschitz_constant} to avoid clutter.

Additionally, we note that in the special case of $\mathcal{X} = \{ \<x>_0 \}$, then \eqref{eq:local_lipschitz_constant} cannot be used in place of \eqref{eq:local_lipschitz_constant_original}. In this case, we can determine that $L(\<x>_0, \mathcal{X}) = 0$ from \eqref{eq:local_lipschitz_constant_original}. So in \eqref{eq:local_lipschitz_constant} we are implicitly assuming that $\mathcal{X} \neq \{ \<x>_0 \}$.


\subsection{Properties of local Lipschitz constants}

There are three properties of local Lipschitz constants that will come in handy in our analysis.
The first is that the local Lipschitz constant taken with respect to set $\mathcal{X}$ is upper bounded by the local Lipschitz constant taken with respect to a superset $\mathcal{S}$ of $\mathcal{X}$. We will use this property in our analysis as we will often determine the local Lipschitz constant with respect to a bound around $\mathcal{X}$ rather than with respect to $\mathcal{X}$ itself.
\begin{prp} \label{prp:superset}
The local Lipschitz constant taken with respect to the set $\mathcal{X}$ is upper bounded by the local Lipschitz constant taken with respect to a superset $\mathcal{S}$ of $\mathcal{X}$:
\begin{equation}
L(\<x>_0, \mathcal{X}) \leq L(\<x>_0, \mathcal{S}) .
\label{eq:local_lipschitz_superset}
\end{equation}
\end{prp}
\begin{proof}
The local Lipschitz constant $L(\<x>_0, \mathcal{S})$ is a supremization over $\mathcal{S}$ which is a superset of $\mathcal{X}$. Since $\mathcal{S}$ contains all elements of $\mathcal{X}$, the supremization over $\mathcal{S}$ results in a value at least as large as the supremization over $\mathcal{X}$, which implies \eqref{eq:local_lipschitz_superset}.
\end{proof}

The second property of local Lipschitz constants is that if all possible inputs of the function $\<f>$ are within a distance $\epsilon$ of the nominal input $\<x>_0$, then all possible outputs of $\<f>$ will be within a distance $\epsilon L(\<x>_0, \mathcal{X})$ of the nominal output $\<f>(\<x>_0)$. We will use this property in our analysis to transfer input bounds from one layer of a network to the next.

\begin{prp} \label{prp:perturbation}
Consider a function $\<f>$, with nominal input $\<x>_0$ and input set $\mathcal{X}$. If there exists an $\epsilon \geq 0$ such that $\norm{\<x> - \<x>_0} \leq \epsilon, \forall \<x> \in \mathcal{X}$, then the deviation of $\<f>(\<x>)$ from $\<f>(\<x>_0)$ is norm-bounded by $\epsilon L(\<x>_0, \mathcal{X})$ for all $\<x> \in \mathcal{X}$:
\begin{equation}
\norm{\<f>(\<x>) - \<f>(\<x>_0)} \leq \epsilon L(\<x>_0, \mathcal{X}), ~~ \forall \<x> \in \mathcal{X} .
\end{equation}
\end{prp}
\begin{proof}
From \eqref{eq:local_lipschitz_constant_original} we have $\norm{\<f>(\<x>) - \<f>(\<x>_0)} \leq L(\<x>_0, \mathcal{X}) \norm{\<x> - \<x>_0}$.
Since $\norm{\<x> - \<x>_0} \leq \epsilon$, then $\norm{\<f>(\<x>) - \<f>(\<x>_0)} \leq \epsilon L(\<x>_0, \mathcal{X})$.
\end{proof}

The third property of local Lipschitz constants is that a composite function is upper bounded by the product of the local Lipschitz constants of each of the composing functions. This property is often applied to global Lipschitz constants to determine global bounds of feedforward networks \cite{Szegedy}. We now present this result for the local case.

\begin{prp} \label{prp:composition}
Consider a function, $\<f>$, with nominal input $\<x>_0$ and input set $\mathcal{X}$. Assume $\<f>$ is the composition of functions $\<g>$ and $\<h>$, i.e., $\<f> = \<h> \circ \<g>$. Define $\<y>_0 = \<g>(\<x>_0)$ and let $\mathcal{Y}$ denote the range of $\<g>$, i.e., $\mathcal{Y} = \{ \<g>(\<x>) ~~ | ~~ \<x> \in \mathcal{X} \}$.
Let $L_{\<f>}$, $L_{\<g>}$ and $L_{\<h>}$ denote the local Lipschitz constants of function $\<f>$, $\<g>$ and $\<h>$, respectively.
The following inequality holds
\begin{equation}
L_{\<f>}(\<x>_0, \mathcal{X}) \leq L_{\<g>}(\<x>_0, \mathcal{X}) L_{\<h>}(\<y>_0, \mathcal{Y}) . 
\label{eq:local_lipschitz_chain_rule}
\end{equation}
\end{prp}
\begin{proof}
We have 
\begin{align}
L_{\<f>}&(\<x>_0, \mathcal{X}) \\
&= \sup_{\substack{\<x> \in \mathcal{X} \\ \<x> \neq \<x>_0}} \frac{\norm{\<f>(\<x>) - \<f>(\<x>_0)}}{\norm{\<x> - \<x>_0}} \\
&= \sup_{\substack{\<x> \in \mathcal{X} \\ \<x> \neq \<x>_0}} \frac{\norm{\<h>(\<g>(\<x>)) - \<h>(\<g>(\<x>_0))}}{\norm{\<x> - \<x>_0}} \\
&= \sup_{\substack{\<x> \in \mathcal{X} \\ \<x> \neq \<x>_0 \\ \<g>(\<x>) \neq \<g>(\<x>_0)}} \frac{\norm{\<g>(\<x>) - \<g>(\<x>_0)}}{\norm{\<x> - \<x>_0}}
\frac{\norm{\<h>(\<g>(\<x>)) - \<h>(\<g>(\<x>_0))}}{\norm{\<g>(\<x>) - \<g>(\<x>_0)}} \\
&\leq \sup_{\substack{\<x> \in \mathcal{X} \\ \<x> \neq \<x>_0}} \frac{\norm{\<g>(\<x>) - \<g>(\<x>_0)}}{\norm{\<x> - \<x>_0}}
\sup_{\substack{\<y> \in \mathcal{Y} \\ \<y> \neq \<y>_0}} \frac{\norm{\<h>(\<y>) - \<h>(\<y>_0)}}{\norm{\<y> - \<y>_0}} \\
&= L_{\<g>}(\<x>_0, \mathcal{X}) L_{\<h>}(\<y>_0, \mathcal{Y}) .
\end{align}




In the derivation above, points such that $\<x> \neq \<x>_0$, $\<g>(\<x>) \neq \<g>(\<x>_0)$, and $\<y> \neq \<y>_0$ can be excluded because these points can only achieve the supremum when $L_{\<f>}(\<x>_0, \mathcal{X}) = 0$, in which case \eqref{eq:local_lipschitz_chain_rule} always holds.
Additionally, in the special cases that $\mathcal{X} = \{ \<x>_0 \}$ or $\mathcal{Y} = \{ \<y>_0 \}$, then $L_{\<f>}(\<x>_0, \mathcal{X}) = 0$ and \eqref{eq:local_lipschitz_chain_rule} always holds.
\end{proof}
Using induction, the relationship in Proposition \ref{prp:composition} can be applied to compositions of more than two functions.
As a result, since a feedforward neural network is a composition of functions, we can bound the local Lipschitz constant of the network by the product of the local Lipschitz constants of each layer.



\section{Local Lipschitz constants of ReLUs} \label{sec:relu_lipschitz}


The rectified linear unit (ReLU) is widely used as an activation function in deep neural networks.
The ReLU is simply the maximum of an input and zero: $\relu(y) = \max(0,y)$ where $\relu : \mathbb{R} \rightarrow \mathbb{R}$.
The ReLU can be applied to a vector by taking the ReLU of each element: $\<relu>(\<y>) = \<max>(\<0>,\<y>)$ where $\<relu> : \mathbb{R}^m \rightarrow \mathbb{R}^m$.
Note that in this section we will write the ReLU as a function of $y$ rather than $x$ to match with our notation in Section \ref{sec:affine_relu_lipschitz}.


\begin{figure}[ht]
\centering
\subfloat[$y_0 \leq 0$]{\includegraphics[width=.20\textwidth]{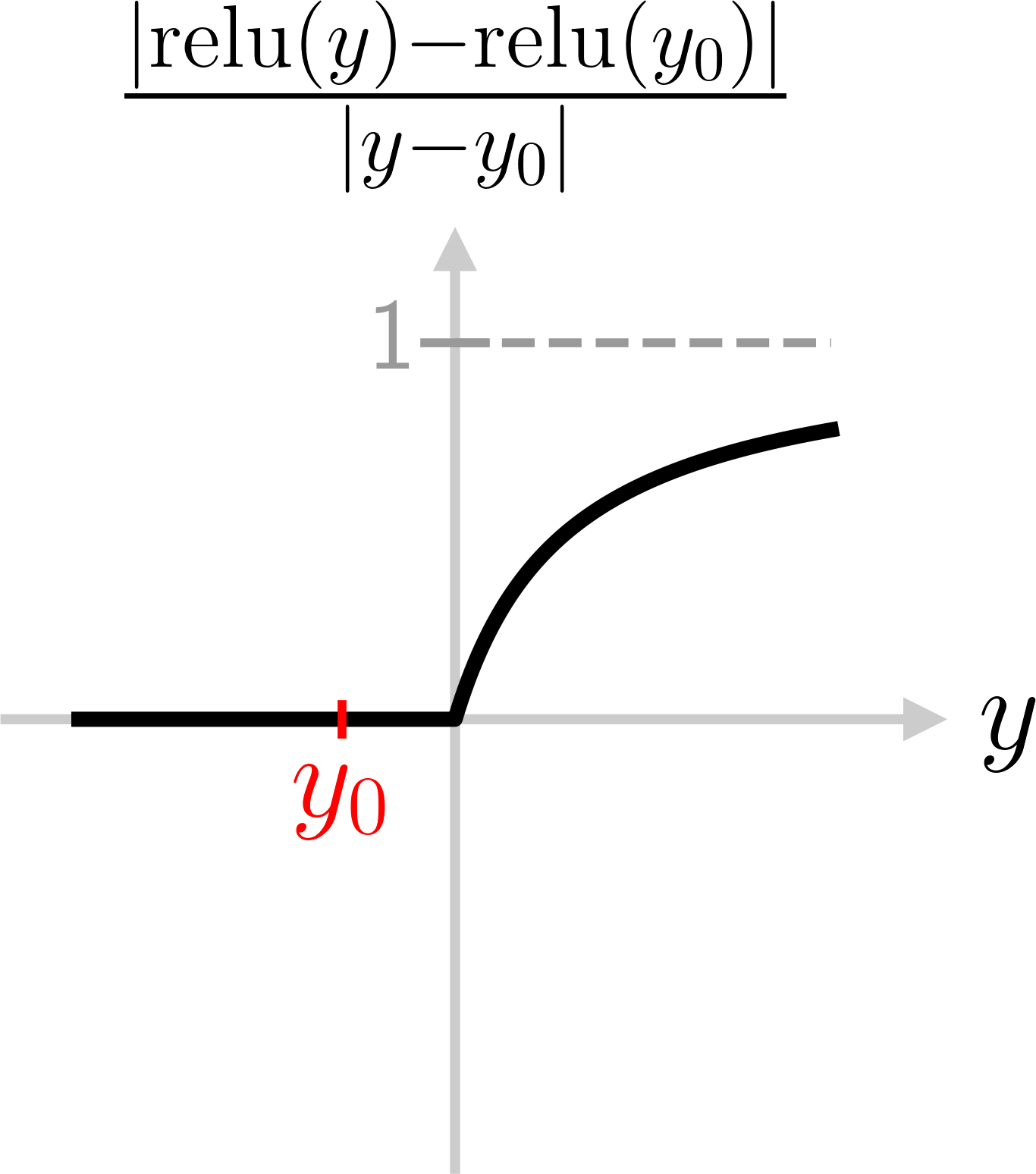}\label{subfig:relu_lipschitz_neg}}
\hspace{20pt}
\subfloat[$y_0 > 0$]{\includegraphics[width=.20\textwidth]{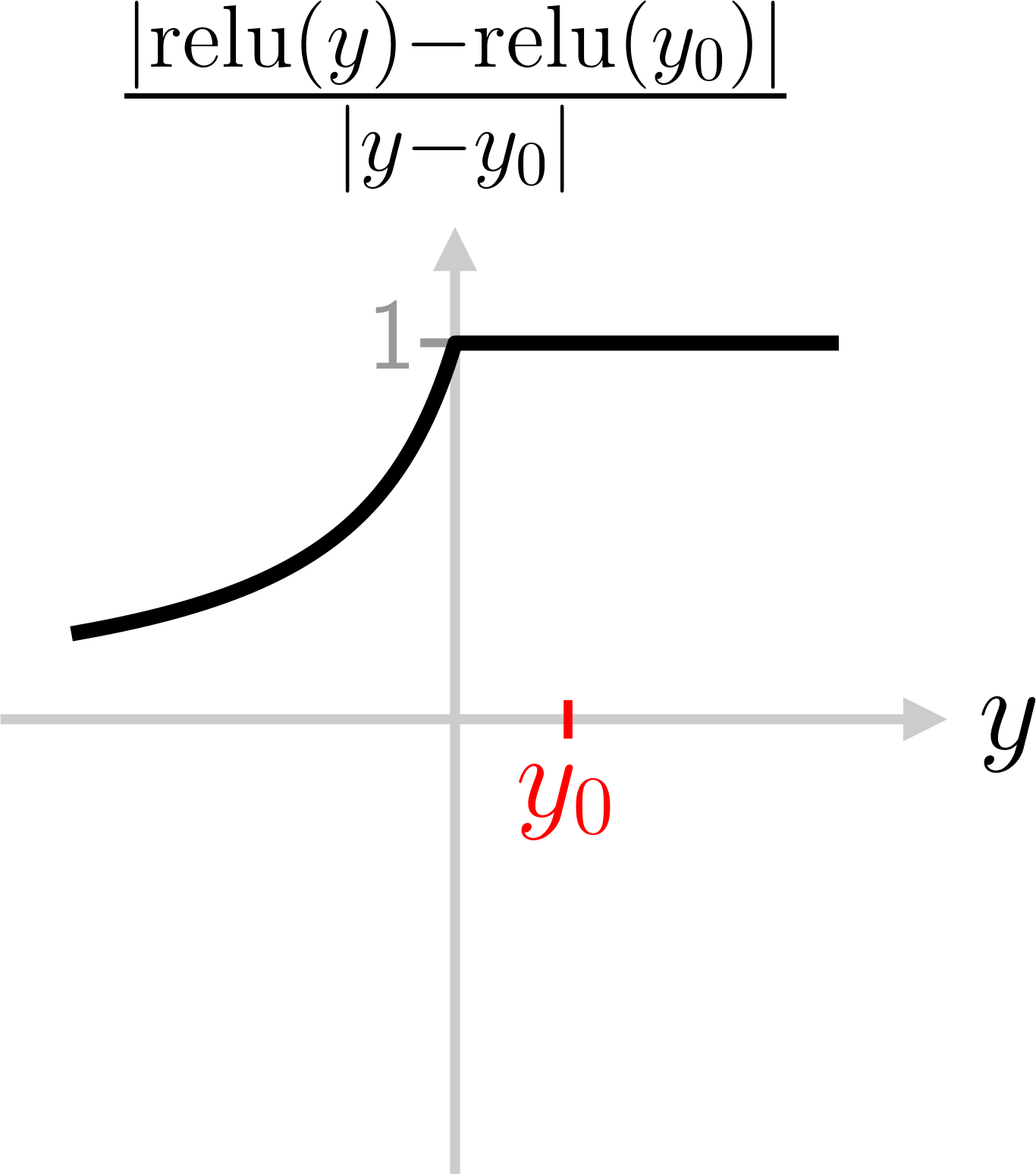}\label{subfig:relu_lipschitz_pos}}
\caption{Diagram showing the ReLU function and its local Lipschitz fraction for (a) negative and (b) positive nominal inputs $y_0$. In both cases, the fraction is non-decreasing, and is therefore maximized at the largest possible $y$. This property is formalized in Theorem \ref{thm:relu}.}
\label{fig:lipschitz_relu}
\end{figure}

\begin{figure*}[ht]
\begin{center}
\includegraphics[width=.90\textwidth]{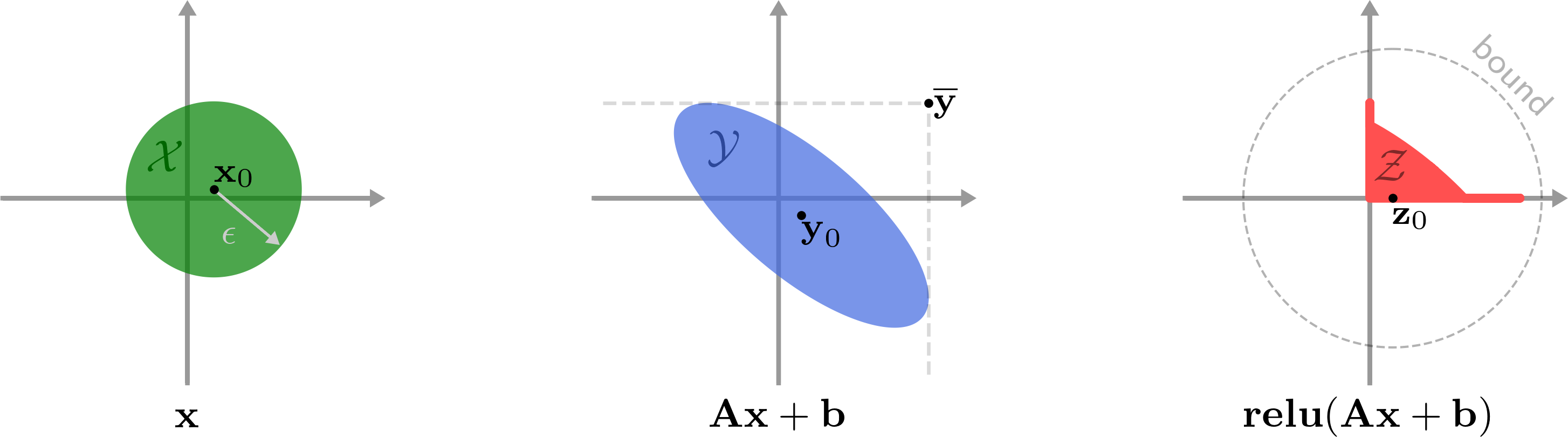}
\end{center}
\caption{Geometric visualization of the set $\mathcal{X}$ (a ball of size $\epsilon$) transformed through affine and ReLU functions. The set $\mathcal{X}$ is transformed by the affine transformation into the set $\mathcal{Y}$ which is transformed by the ReLU into the set $\mathcal{Z}$. The variable $\<x>_0$ is the nominal input which is transformed into $\<y>_0$ by the affine function and then into $\<z>_0$ by the ReLU function. The ``bound'' refers to a norm bound on the output of the ReLU, which in this paper is computed by Theorem \ref{thm:aff_relu} and has magnitude $\epsilon \norm{\<R> \<A> \<D>}$.}
\label{fig:transformation}
\end{figure*}

We will now determine the local Lipschitz constant of the scalar ReLU function. This result will be used in bounding the local Lipschitz constant of the affine-ReLU function.
\begin{thm} \label{thm:relu}
Consider the ReLU function of a scalar value $y \in \mathbb{R}$. Let $y_0 \in \mathbb{R}$ denote the nominal input, let $\mathcal{Y} \subset \mathbb{R}$ denote the set of permissible inputs, and let $\bar{y} \in \mathbb{R}$ denote the largest element of $\mathcal{Y}$ such that $\bar{y} \neq y_0$. The local Lipschitz constant of the ReLU function is
\begin{align}
L(y_0,\mathcal{Y}) =
\begin{cases}
0, & \mathcal{Y} = \{ y_0 \} \\
\frac{\relu(\bar{y}) - \relu(y_0)}{\bar{y} - y_0}, & \text{otherwise.} \\
\end{cases}
\label{eq:relu_local_lipschitz_constant}
\end{align}
\end{thm}
\begin{proof}
The first case in \eqref{eq:relu_local_lipschitz_constant} follows from the fact that the local Lipschitz constant of any function is zero if the nominal input is the only permissible input.
In all other cases, applying the definition of the local Lipschitz constant from \eqref{eq:local_lipschitz_constant} to this scenario yields
\begin{align}
L(y_0,\mathcal{Y}) = \sup_{y \in \mathcal{Y}} \frac{\abs{\relu(y) - \relu(y_0)}}{\abs{y - y_0}} .
\label{eq:local_lipschitz_relu_def}
\end{align}
We refer to the fraction in the RHS of \eqref{eq:local_lipschitz_relu_def} as the ``local Lipschitz fraction''. The table below shows what the local Lipschitz fraction evaluates to for different signs of $y$ and $y_0$:
\begin{center}
\begin{tabular}{ c c c c c }
\toprule
$y_0$ & $y$ & $\relu(y_0)$ & $\relu(y)$ & $\frac{\abs{\relu(y) - \relu(y_0)}}{\abs{y - y_0}}$ \\ 
\midrule
$\leq 0$ & $\leq 0$ & 0 & 0 & 0 \\
$\leq 0$ & $> 0$ & 0 & $y$ & $y/(y - y_0)$ \\
$> 0$ & $\leq 0$ & $y_0$ & 0 & $y_0/(y_0 - y)$ \\
$> 0$ & $> 0$ & $y_0$ & $y$ & 1 \\
\bottomrule
\end{tabular}
\end{center}

Using this table, we can break this problem down into the following two cases, and show that for each case the local Lipschitz fraction is non-decreasing in $y$, which implies that it is maximized at $y = \bar{y}$.

\textit{Case 1: $y_0 \leq 0$ (Fig. \ref{subfig:relu_lipschitz_neg})}: The local Lipschitz fraction equals 0 for non-positive $y$ and equals $y/(y-y_0)$ for positive $y$. The derivative of the latter expression with respect to $y$ is $-y_0/(y-y_0)^2$ which is non-negative since $y_0 \leq 0$.
So the local Lipschitz fraction is non-decreasing in $y$ which implies it is maximized at the largest possible $y \neq y_0$ (i.e., $\bar{y}$), so \eqref{eq:relu_local_lipschitz_constant} holds in this case.

\textit{Case 2: $y_0 > 0$ (Fig. \ref{subfig:relu_lipschitz_pos})}: The local Lipschitz fraction equals 1 for positive $y$ and equals $y_0/(y_0 - y)$ for non-positive $y$. The derivative of the latter expression with respect to $y$ is $y_0/(y_0-y)^2$ which is positive since $y_0 > 0$.
So, the local Lipschitz fraction is non-decreasing in $y$ which implies it is maximized at the largest possible $y \neq y_0$ (i.e., $\bar{y}$), so \eqref{eq:relu_local_lipschitz_constant} holds in this case as well.

\end{proof}
\section{Local Lipschitz bounds for affine-ReLU functions} \label{sec:affine_relu_lipschitz}

\subsection{Affine-ReLU functions} \label{sec:affine_relu_functions}

Affine functions are ubiquitous in neural networks as convolution, fully-connected, and normalization operations are all affine.
An affine function can be written as
\begin{equation}
\<y> = \<A> \<x> + \<b>
\label{eq:y}
\end{equation}
where $\<A> \in \mathbb{R}^{m \times n}$, $\<x> \in \mathbb{R}^n$, and $\<b> \in \mathbb{R}^m$.
The inputs and outputs of affine functions in neural networks are often multi-dimensional arrays, but it is mathematically equivalent to consider them to be 1D vectors. So in this paper, $\<x>$ will often represent a multi-dimensional array that has been been reshaped into a 1D vector. Note that the global Lipschitz constant of an affine function is $\norm{\<A>}$.

We define an affine-ReLU function as a ReLU composed with an affine function, which can be written as
\begin{equation}
\<z> = \<relu>(\<A> \<x> + \<b>) .
\end{equation}
In a neural network, affine-ReLU functions represent one layer (e.g., convolution with a ReLU activation). Note that although they are commonly used in neural networks, we are not aware of any work that has directly analyzed affine-ReLU functions except for \cite{Dittmer}.

We denote $\<x>_0$ as the nominal input to the affine-ReLU function, $\<y>_0$ as the affine transformation of $\<x>_0$, and $\<z>_0$ as the ReLU transformation of $\<y>_0$:
\begin{align}
\begin{aligned}
\<y>_0 &\coloneqq \<A> \<x>_0 + \<b> \\
\<z>_0 &\coloneqq \<relu>(\<y>_0) = \<relu>(\<A> \<x>_0 + \<b>) .
\label{eq:y0_z0}
\end{aligned}
\end{align}
Recall that $\mathcal{X}$ denotes the domain of inputs, so we denote the range of the affine function as $\mathcal{Y}$, and the range of the ReLU function as $\mathcal{Z}$ (see Fig. \ref{fig:transformation}):
\begin{align}
\begin{aligned}
\mathcal{Y} &\coloneqq \{ \<A> \<x> + \<b> ~~ | ~~ \<x> \in \mathcal{X} \} \\
\mathcal{Z} &\coloneqq \{ \<relu>(\<y>) ~~ | ~~ \<y> \in \mathcal{Y} \} .
\label{eq:Y_Z}
\end{aligned}
\end{align}

\subsection{Input set} \label{sec:input_set}

In this paper, we will consider the input set $\mathcal{X}$ to be in a specific mathematical form.
We will use this form not just for affine-ReLU functions, but for other layers (e.g., max pooling) of a network as well.

More specifically, we consider the set $\mathcal{X}$ to be a ball of perturbations centered at the nominal input $\<x>_0$, for which certain entries of all vectors $\<x> \in \mathcal{X}$ are equal to zero. Knowledge of these entries will come from having determined that specific ReLUs of previous layers cannot be activated. Since certain dimensions of $\<x> \in \mathcal{X}$ are zero, we can think of $\mathcal{X}$ as a lower dimensional Euclidean ball embedded in a higher dimensional space (e.g., a 2D disk in three-dimensional space). We also denote perturbations about $\<x>_0$ as $\Delta \<x> \in \mathbb{R}^n$, and the maximum magnitude of these perturbations as $\epsilon \in \mathbb{R}$.

We will use a diagonal binary matrix $\<D> \in \mathbb{R}^{n \times n}$ to ensure the zero elements of $\mathcal{X}$ are enforced. We will often refer to this matrix as a ``domain-restriction matrix''.
Letting $x_i$ denote the $i^{th}$ entry of $\<x>$, we have
\begin{align}
d_i &\coloneqq
\begin{cases}
0, ~~~ x_i = 0 ~~~ \forall \<x> \in \mathcal{X} \\
1, ~~~ \text{otherwise}
\end{cases} \\
\<D> &\coloneqq \<diag>(d_1,\cdots,d_n)
\label{eq:di_D}
\end{align}
where $\<diag>: \mathbb{R}^n \rightarrow \mathbb{R}^{n \times n}$ forms a diagonal matrix from its inputs.
Note that if we do not know any input indices to be zero, then $\<D>$ will be the identity matrix.

Using $\<D>$ and the $\epsilon$ norm constraint, we can write the input $\<x>$ and input set $\mathcal{X}$ as
\begin{align}
\begin{aligned}
\<x> &= \<x>_0 + \<D> \Delta \<x> \\
\mathcal{X} &= \{ \<x>_0 + \<D> \Delta \<x> ~~ | ~~ \norm{\Delta \<x>} \leq \epsilon \} .
\end{aligned}
\label{eq:x_X}
\end{align}

\subsection{Upper bound on the affine function}

We can now substitute $\<x>$ and $\mathcal{X}$ from \eqref{eq:x_X} into the equations for $\<y>$ and $\mathcal{Y}$ from \eqref{eq:y} and \eqref{eq:Y_Z}. Letting $y_i$ and $y_{0,i}$ denote the $i^{th}$ elements of $\<y>$ and $\<y>_0$, respectively, and letting $\<a>_i^T \in \mathbb{R}^n$ denote the $i^{th}$ row of $\<A>$, we have
\begin{align}
\begin{aligned}
\<y> &= \<A> \<D> \Delta \<x> + \<y>_0 \\
y_i &= \<a>_i^T \<D> \Delta \<x> + y_{0,i} \\
\mathcal{Y} &= \{ \<A> \<D> \Delta \<x> + \<y>_0 ~~ | ~~ \norm{\Delta \<x>} \leq \epsilon \} .
\label{eq:y_yi_Y}
\end{aligned}
\end{align}

We now discuss how to compute an elementwise upper bound on $\mathcal{Y}$ (see center of Fig. \ref{fig:transformation}). First, for each $i$, we define the set of all $y_i$ values as
\begin{equation}
\mathcal{Y}_i \coloneqq \{ y_i ~~ | ~~ \<y> \in \mathcal{Y} \} .
\label{eq:Yi}
\end{equation}
Next, define the upper bound $\bar{y}_i$ and vector $\overline{\<y>}$ as follows:
\begin{align}
\bar{y}_i &\coloneqq \max_{y_i \in \mathcal{Y}_i} y_i \label{eq:overline_yi} \\
\overline{\<y>} &\coloneqq [\bar{y}_1 ~~ \cdots ~~ \bar{y}_m ]^T \in \mathbb{R}^m \label{eq:overline_y} .
\end{align}
The following proposition describes how to compute $\overline{\<y>}$.
\begin{prp} \label{prp:yi_bar}
Consider the upper bound $\bar{y}_i$ from \eqref{eq:overline_yi} where $\mathcal{Y}_i$ is given by \eqref{eq:Yi}, and $y_i$ and $\mathcal{Y}$ are given by \eqref{eq:y_yi_Y}. The equation for $\bar{y}_i$ is
\begin{align}
\bar{y}_i = \epsilon \norm{\<a>_i^T \<D>} + y_{0,i} .
\label{eq:ybar_i}
\end{align}
\end{prp}
\begin{proof}
Substituting \eqref{eq:y_yi_Y} into \eqref{eq:overline_yi} yields
\begin{align}
\bar{y}_i &= \max_{\norm{\Delta \<x>} \leq \epsilon} ~ \<a>_i^T \<D> \Delta \<x> + y_{0,i} .
\label{eq:yi_maximization}
\end{align}
In this maximization, we can ignore $y_{0,i}$ ignored because it is constant. The remaining term is the dot product of $\<D> \<a>_i$ and $\Delta \<x>$, which will be maximized when $\Delta \<x>$ points in the direction of $\<D> \<a>_i$ and has maximum magnitude of $\epsilon$.
Therefore, the value of $\Delta \<x>$ that maximizes \eqref{eq:yi_maximization} is $\epsilon \<D> \<a>_i / \norm{\<D> \<a>_i}$. Plugging this expression into the RHS of \eqref{eq:yi_maximization} yields \eqref{eq:ybar_i}.
\end{proof}

\subsection{Local Lipschitz constant upper bound}

We can now derive a bound on the local Lipschitz constant of an affine-ReLU function.
\begin{thm} \label{thm:aff_relu}
Consider the affine-ReLU function $\<relu>(\<A> \<x> + \<b>)$ with nominal input $\<x>_0$, and input set $\mathcal{X}$ from \eqref{eq:x_X} with domain-restriction matrix $\<D>$ from \eqref{eq:di_D}.
Let $\<y>$ and $\<y>_0$ denote the output and nominal output of the affine function as in \eqref{eq:y} and \eqref{eq:y0_z0}, respectively.
Let $y_i$ and $y_{0,i}$ denote the $i^{th}$ element of $\<y>$ and $\<y>_0$, respectively.
Let $\mathcal{Y}_i$ denote the set of all $y_i$ values for $\<y> \in \mathcal{Y}$ as in \eqref{eq:Yi}, and let $\bar{y}_i$ denote the maximum of all $y_i \in \mathcal{Y}_i$ as in \eqref{eq:overline_yi}. 
Define $r_i$ and $\<R>$ as follows:
\begin{align}
\begin{aligned}
r_i &\coloneqq 
\begin{cases}
0, & \mathcal{Y}_i = \{ y_{0,i} \} \\
\frac{\relu(\bar{y}_i) - \relu(y_{0,i})}{\bar{y}_i - y_{0,i}}, & \text{otherwise} \\
\end{cases} \\
\<R> &\coloneqq \<diag>(r_1, \cdots, r_m) .
\label{eq:ri_R}
\end{aligned}
\end{align}
The following is an upper bound on the affine-ReLU function's local Lipschitz constant:
\begin{equation}
L( \<x>_0, \mathcal{X} ) \leq \norm{\<R> \<A> \<D>} .
\label{eq:affine_relu_Lipschitz_bound}
\end{equation}
\end{thm}
\begin{proof}
We start by considering the output of the affine function on an elementwise basis. We can use $y_i$, $y_{0,i}$, $\mathcal{Y}_i$, and $\bar{y}_i$ in place of $y$, $y_0$, $\mathcal{Y}$, and $\bar{y}$ (respectively) in Theorem \ref{thm:relu}.
Applying the theorem, we can see that $r_i$ in \eqref{eq:ri_R} is the local Lipschitz constant from \eqref{eq:relu_local_lipschitz_constant}, so using \eqref{eq:local_lipschitz_constant_original} we have
\begin{equation}
r_i \abs{y_i - y_{0,i}} \geq \abs{\relu(y_i) - \relu(y_{0,i})} .
\label{eq:ri_inequality}
\end{equation}
Note that in \eqref{eq:overline_yi} we have not assumed that $y_i \neq y_{0,i}$. We do not have to make this assumption because using the expression for $y_i$ from \eqref{eq:y_yi_Y} we can see that since $\norm{\Delta \<x>} \leq \epsilon$, then $\mathcal{Y}_i$ will always be an interval centered at $y_{0,i}$, so the only way the maximum $y_i$ will be $y_{0,i}$ is when $\mathcal{Y}_i = \{ y_{0,i} \}$, in which case $r_i$ is computed without $\bar{y}_i$.

Taking \eqref{eq:ri_inequality} and stacking it into an elementwise vector equation for all $i$ yields
\begin{equation}
\abs{\<relu>(\<A> \<x> + \<b>) - \<relu>(\<A> \<x>_0 + \<b>)} \leq \<R> \abs{\<A> \<x> + \<b> - (\<A> \<x>_0 + \<b>)} .
\end{equation}
Noting that the inequality above holds elementwise, and that $\<R>$ is a diagonal matrix with non-negative entries, we have
\begin{equation}
\norm{\<relu>(\<A> \<x> + \<b>) - \<relu>(\<A> \<x>_0 + \<b>)} \leq \norm{\<R> (\<A> \<x> + \<b> - (\<A> \<x>_0 + \<b>))} .
\end{equation}

Next, we substitute the equation above into the definition of the local Lipschitz constant from \eqref{eq:local_lipschitz_constant} as follows:
\begin{align}
L \left( \<x>_0, \mathcal{X} \right)
&= \sup_{\<x> \in \mathcal{X}} \frac{\norm{\<relu>(\<A> \<x> + \<b>) - \<relu>(\<A> \<x>_0 + \<b>)}}{\norm{\<x> - \<x>_0}} \\
&\leq \sup_{\<x> \in \mathcal{X}} \frac{\norm{\<R> (\<A> \<x> + \<b> - (\<A> \<x>_0 + \<b>))}}{\norm{\<x> - \<x>_0}} \\ 
&= \sup_{\<x> \in \mathcal{X}} \frac{\norm{\<R> (\<A> \<x> - \<A> \<x>_0)}}{\norm{\<x> - \<x>_0}} \\ 
&= \max_{\norm{\Delta \<x>} \leq \epsilon} \frac{\norm{\<R> \<A> \<D> \Delta \<x>}}{\norm{\<D> \Delta \<x>}} \\ 
&= \max_{\norm{\Delta \<x>} \leq \epsilon} \frac{\norm{\<R> \<A> \<D> \<D> \Delta \<x>}}{\norm{\<D> \Delta \<x>}} \\ 
&\leq \max_{\norm{\Delta \<x>} \leq \epsilon} \frac{\norm{\<R> \<A> \<D>} \norm{\<D> \Delta \<x>}}{\norm{\<D> \Delta \<x>}} \\ 
&= \norm{\<R> \<A> \<D>} .
\end{align}
In the derivation above, we have used \eqref{eq:x_X} and the fact that $\<D> = \<D> \<D>$ since $\<D>$ is a diagonal binary matrix
\end{proof}

Note that since $\<R>$ and $\<D>$ are both diagonal binary matrices, it follows from Theorem \ref{thm:aff_relu} that $L( \<x>_0, \mathcal{X} ) \leq \norm{\<A>}$. This result provides a global Lipschitz bound that was mentioned as early as \cite{Szegedy}.
In this paper, we will often refer to network-wide bounds found using $\norm{\<A>}$ as the ``global'' bound.

\section{Lipschitz constants of max pooling functions} \label{sec:max_pooling_lipschitz}

\subsection{Max pooling functions}

Max pooling is ubiquitous in neural networks, so in order to compute the local Lipschitz constant of a full network, we need to compute the local or global Lipschitz constant of the max pooling function. 

\begin{figure}[ht]
\centering
\begin{tikzpicture}[every text node part/.style={align=center}]
\fill[yellow!15!white] (-4.4,-1.9) rectangle (4.4,1.9);
\node[anchor=center] (input) at (0,0)
	{\resizebox{.48\textwidth}{!}{
		$
	\<M>(\<x>) \<x> =
	\begin{bmatrix}
	{\color{red} 0} & {\color{red} 0} & {\color{red} 1} & 0 & 0 & 0 & 0 \\
	0 & 0 & {\color{red} 1} & {\color{red} 0} & {\color{red} 0} & 0 & 0 \\
	0 & 0 & 0 & 0 & {\color{red} 1} & {\color{red} 0} & {\color{red} 0}
	\end{bmatrix}
	\begin{bmatrix}
	1 \\ 1 \\ 4 \\ 2 \\ 3 \\ 2 \\ 1
	\end{bmatrix}
	=
	\begin{bmatrix}
	4 \\ 4 \\ 3
	\end{bmatrix}
	$}
	};

\draw [blue,thick,|-|] (-2.22,1) -- (-0.5,1);
\draw [blue,thick,|-|] (-0.52,1) -- (0.75,1);
\node[blue] at (-1.4,1.5) {\normalsize size\\[-2em] \normalsize kernel};
\node[blue] at (0.1,1.5) {\normalsize size\\[-2em] \normalsize stride};

\end{tikzpicture}


\caption{A 1D max pooling operation with kernel size $k=3$ and stride size $s=2$, shown in matrix form. Note that max pooling can be mathematically described as a piecewise linear matrix operation. Each row in $\<M>(\<x>)$ corresponds to one pooling region (and one output element), and each column corresponds to one input element. In this diagram, the red elements represent the elements in each pooling region.
}
\label{fig:maxpool_matrix}
\end{figure}

Max pooling is a downsampling operation which involves sliding a small window, which we will call a ``kernel'', across an input array, and taking the maximum value within the kernel every time the kernel is in place. The amount that the kernel is shifted in each direction is called the ``stride''. For each placement of the kernel, the elements of the input array within the kernel are called the ``pooling region''. If the stride size is less than the kernel size in any dimension, then some input elements will be part of more than one pooling region, and the max pooling function is called ``overlapping''. Otherwise, each input element will appear in a maximum of one pooling region, and the function is called ``non-overlapping''.

Although max pooling functions typically operate on multi-dimensional arrays, as we did with affine functions in Section \ref{sec:affine_relu_functions}, without loss of generality we can consider the input and output of the max pooling function to be vectors. Max pooling is piecewise linear operation and can be written as $\<f>(\<x>) = \<M>(\<x>) \<x>$ where $\<M>(\<x>) \in \mathbb{R}^{m \times n}$ is a binary matrix (see Fig. \ref{fig:maxpool_matrix}).

\subsection{Lipschitz constants of piecewise linear functions}

Computing the Lipschitz constant of a max pooling function requires the following result for the Lipschitz constant of a vector-valued continuous piecewise linear function.

\begin{figure}[ht]
\centering
\includegraphics[width=.45\textwidth]{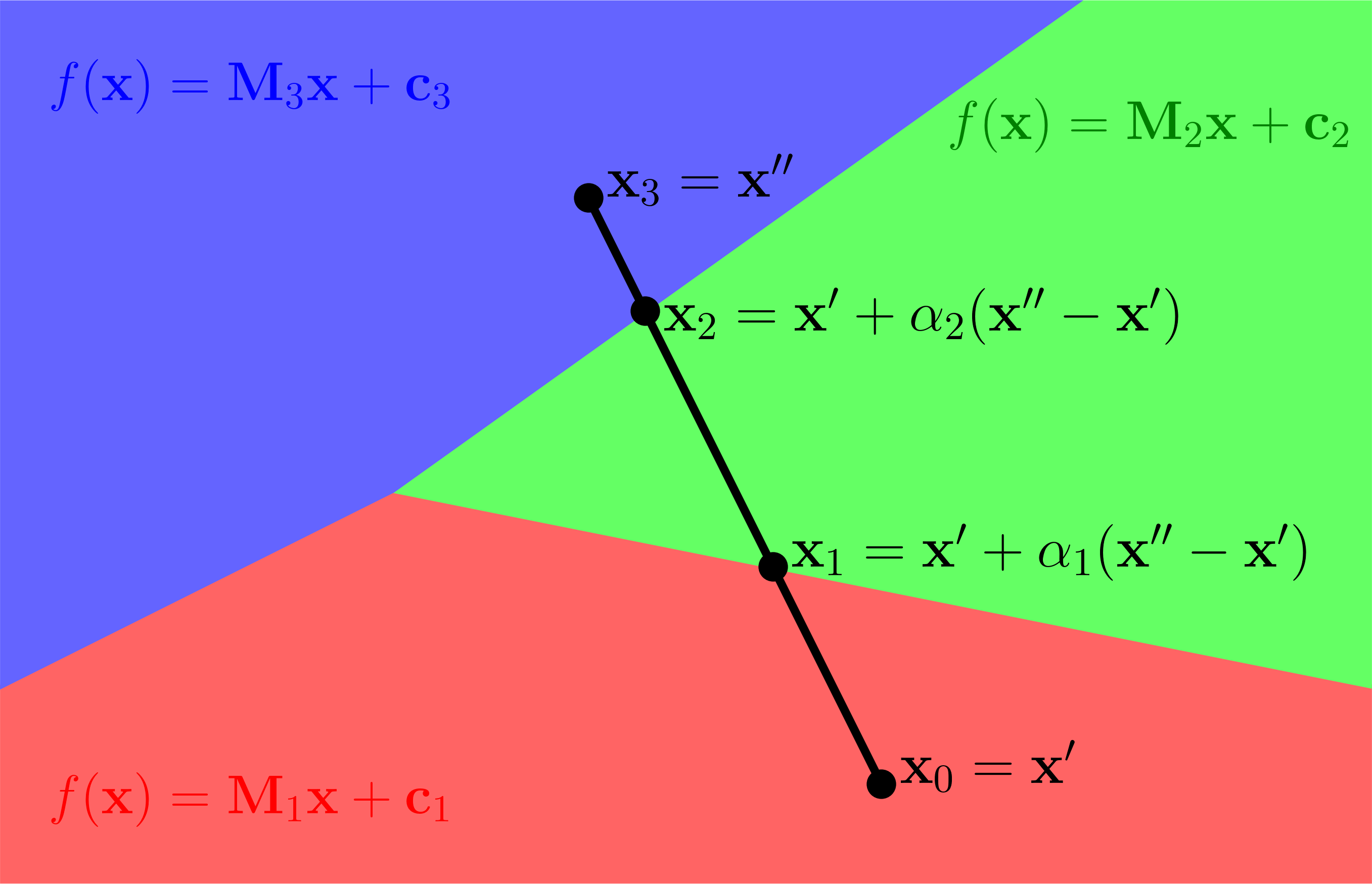}
\caption{Diagram of the domain of a piecewise linear function in 2D. The different colors represent different linear regions, each of which has a constant scaling matrix, $\<M>$ and bias vector, $\<c>$. The points $\<x>'$ and $\<x>''$ are the start and end points of the line segment, and $\<x>_1$ and $\<x>_2$ are the internal points on the line segment for which the linear regions change.}
\label{fig:piecewise}
\end{figure}

\begin{lem} \label{lem:piecewise_linear_lipschitz}
Consider a vector-valued, continuous piecewise linear function $\<f>(\<x>) = \<M>(\<x>) \<x> + \<c>(\<x>)$ where $\<M>(\<x>) \in \mathbb{R}^{m \times n}$ and $\<c>(\<x>) \in \mathbb{R}^m$. The global Lipschitz constant of $\<f>$ is the maximum norm of all matrices $\<M>(\<x>)$:
\begin{equation}
L = \max_{\<x>} \norm{\<M>(\<x>)} . 
\end{equation}
\end{lem}
\begin{proof}
Let $\<x>' \in \mathbb{R}^n$ and $\<x>'' \in \mathbb{R}^n$ denote two arbitrary points. Using Fig. \ref{fig:piecewise} as a guide, a line segment from $\<x>'$ to $\<x>''$ will travel through some number $p$ different linear regions, which we will reference using indices $i=1,...,p$.
Denote the associated scaling matrix and bias of linear region $i$ as $\<M>_i$ and $\<c>_i$ respectively. Define $\Delta \<x> \coloneqq \<x>'' - \<x>'$ and denote the points on the boundary of the linear regions as $\<x>_i \coloneqq \<x>' + \alpha_i \Delta \<x>$ where $0 \leq \alpha_i \leq 1$. Let $\alpha_0 = 0$ and $\alpha_p = 1$ so that $\<x>_0 = \<x>'$ and $\<x>_p = \<x>''$, and also define $\Delta \alpha_i \coloneqq \alpha_i - \alpha_{i-1}$ for $i=1,..,p$. Note that $\sum_{i=1}^p \Delta \alpha_i = 1$. The difference in the function $\<f>$ across the $i^{th}$ linear region of the line segment can then be written as 
\begin{align}
\<f>(\<x>_i) - \<f>(\<x>_{i-1}) &= \<M>_i \<x>_i + \<c>_i - (\<M>_i \<x>_{i-1} + \<c>_i) \\
&= \<M>_i (\<x>_i - \<x>_{i-1}) \\
&= \Delta \alpha_i \<M>_i \Delta \<x> .
\end{align}
We can write $\<f>(\<x>'') - \<f>(\<x>')$ as the sum of differences in $\<f>$ across each linear region. Noting that $\<f>(\<x>_0) = \<f>(\<x>')$ and $\<f>(\<x>_p) = \<f>(\<x>'')$ we have 
\begin{align}
\<f>(\<x>'') - \<f>(\<x>') &= \sum_{i=1}^p \<f>(\<x>_i) - \<f>(\<x>_{i-1}) \\
&= \sum_{i=1}^p \Delta \alpha_i \<M>_i \Delta \<x> .
\end{align}
Next, we plug the equation above into the definition of the Lipschitz constant in \eqref{eq:global_lipschitz_constant}. Note that in \eqref{eq:global_lipschitz_constant}, the points $\<x>_1$ and $\<x>_2$ denote any points in $\mathbb{R}^n$, but in this theorem we are using $\<x>'$ and $\<x>''$ to denote any points in $\mathbb{R}^n$, and $\<x>_1$ and $\<x>_2$ to denote points on the line segment between $\<x>'$ and $\<x>''$:
\begin{align}
L &= \sup_{\<x>' \neq \<x>''} \frac{\norm{\<f>(\<x>'') - \<f>(\<x>')}}{\norm{\<x>'' - \<x>'}} \\
&= \sup_{\<x>' \neq \<x>''} \frac{\norm{\sum_{i=1}^p \Delta \alpha_i \<M>_i \Delta \<x>}}{\norm{\Delta \<x>}} \\
&\leq \sup_{\<x>' \neq \<x>''} \frac{\sum_{i=1}^p \norm{\Delta \alpha_i \<M>_i \Delta \<x>}}{\norm{\Delta \<x>}} \\
&\leq \sup_{\<x>' \neq \<x>''} \frac{\sum_{i=1}^p \Delta \alpha_i \norm{\<M>_i} \norm{\Delta \<x>}}{\norm{\Delta \<x>}} \\
&= \sup_{\<x>' \neq \<x>''} \sum_{i=1}^p \Delta \alpha_i \norm{\<M>_i} \\
&\leq \max_{\<x>} \norm{\<M>(\<x>)} .
\end{align}
In the last step we used the fact that $\sum_{i=1}^p \Delta \alpha_i = 1$.

We have shown that $L \leq \max_{\<x>} \norm{\<M>(\<x>)}$ so we can complete the proof by showing that $L \geq \max_{\<x>} \norm{\<M>(\<x>)}$. Let $\<M>^*$ denote the $\<M>$ matrix of the linear region associated with $\max_{\<x>} \norm{\<M>(\<x>)}$, i.e., $\max_{\<x>} \norm{\<M>(\<x>)} = \norm{\<M>^*}$. Let $\<x>'^*$ and $\<x>''^*$ denote any two points in the linear region of the $\<M>^*$ such that $\norm{\<M>^*} = \norm{\<M>(\<x>''^* - \<x>'^*)}/\norm{\<x>''^* - \<x>'^*}$.
We have
\begin{align}
L &= \sup_{\<x>' \neq \<x>''} \frac{\norm{\<f>(\<x>'') - \<f>(\<x>')}}{\norm{\<x>'' - \<x>'}} \\
&\geq \frac{\norm{\<f>(\<x>''^*) - \<f>(\<x>'^*)}}{\norm{\<x>''^* - \<x>'^*}} \\
&= \frac{\norm{\<M>^*(\<x>''^* - \<x>'^*)}}{\norm{\<x>''^* - \<x>'^*}} \\
&= \norm{\<M>^*} \\
&= \max_{\<x>} \norm{\<M>(\<x>)} .
\end{align}
We have shown that $L \leq \max_{\<x>} \norm{\<M>(\<x>)}$ and $L \geq \max_{\<x>} \norm{\<M>(\<x>)}$ so $L = \max_{\<x>} \norm{\<M>(\<x>)}$.
\end{proof}
We will use this lemma to determine the Lipschitz constant of a max pooling function. Note that ReLUs are also piecewise linear, so we could use this result to determine the global Lipschitz constant of the ReLU (which equals one).

\subsection{Lipschitz constants of max pooling functions}

If a max pooling function is non-overlapping, then it will have a Lipschitz constant of one. However, if it is overlapping, the Lipschitz constant will be larger due to the fact that an input element can map to multiple places in the output (note that this fact is sometimes overlooked in the literature). We present the global Lipschitz constant of a general overlapping or non-overlapping max pooling function in the following theorem.

\begin{thm} \label{thm:max_pool}
Consider a max pooling function. Let $n_{\text{max}}$ denote the maximum number of pooling regions that any input element is part of.
The global Lipschitz constant of the max pooling function is:
\begin{equation}
L =
\sqrt{n_{\text{max}}} .
\label{eq:max_pooling_lipschitz}
\end{equation}
\end{thm}
\begin{proof}
Without loss of generality, we can consider the input and output of the max pooling function to be vectors in $\mathbb{R}^n$ and $\mathbb{R}^m$, respectively, so we can represent max pooling as a matrix operation. The max pooling function is piecewise linear and can be written as $\<f>(\<x>) = \<M>(\<x>) \<x>$ where $\<M>(\<x>) \in \mathbb{R}^{m \times n}$ is a binary matrix (see Fig. \ref{fig:maxpool_matrix}).

Let $\<m>_j(\<x>) \in \mathbb{R}^m$ denote the $j^{th}$ column of $\<M>$ and let $m_{ij}(\<x>) \in \mathbb{R}$ denote the $(i,j)^{th}$ entry of $\<M>(\<x>)$. Dropping the explicit dependence on $\<x>$, we have $\<M> = [\<m>_1 ~~ \cdots ~~ \<m>_n]$. Each row of $\<M>$ represents one pooling region, and will only have a single 1, with all other values being zero. Each column of $\<M>$ represents a particular input, and the number of occurrences of the value 1 in any column represents the number of pooling regions that input is the maximum for. Since any input can be the maximum for all of its pooling regions, the maximum possible number occurrences of the value 1 in any column of all matrices $\<M>(\<x>)$ can be expressed as
\begin{align}
\begin{aligned}
n_{\text{max}}
&= \max_{\<x>} \left( \max_j \<m>_j^T \<m>_j \right) .
\end{aligned}
\label{eq:n_max_M}
\end{align}




Since each row contains a single 1, the columns $\<m>_j$ are orthogonal
so $\<M>^T \<M> = \<diag>(\<m>_1^T \<m>_1, \cdots , \<m>_n^T \<m>_n)$. Since $\<M>^T \<M>$ is diagonal, its singular values are $\<m>_j^T \<m>_j$ which implies $\norm{\<M>^T \<M>}{=}\max_j \<m>_j^T \<m>_j$ and $\norm{\<M>}{=}\max_j \sqrt{\<m>_j^T \<m>_j}$.
Using \eqref{eq:n_max_M} and Lemma \ref{lem:piecewise_linear_lipschitz}, we have $L = \max_{\<x>} \norm{\<M>(\<x>)} = \sqrt{n_{\text{max}}}$.


\end{proof}

\begin{figure}[ht]
\centering
\includegraphics[width=.30\textwidth]{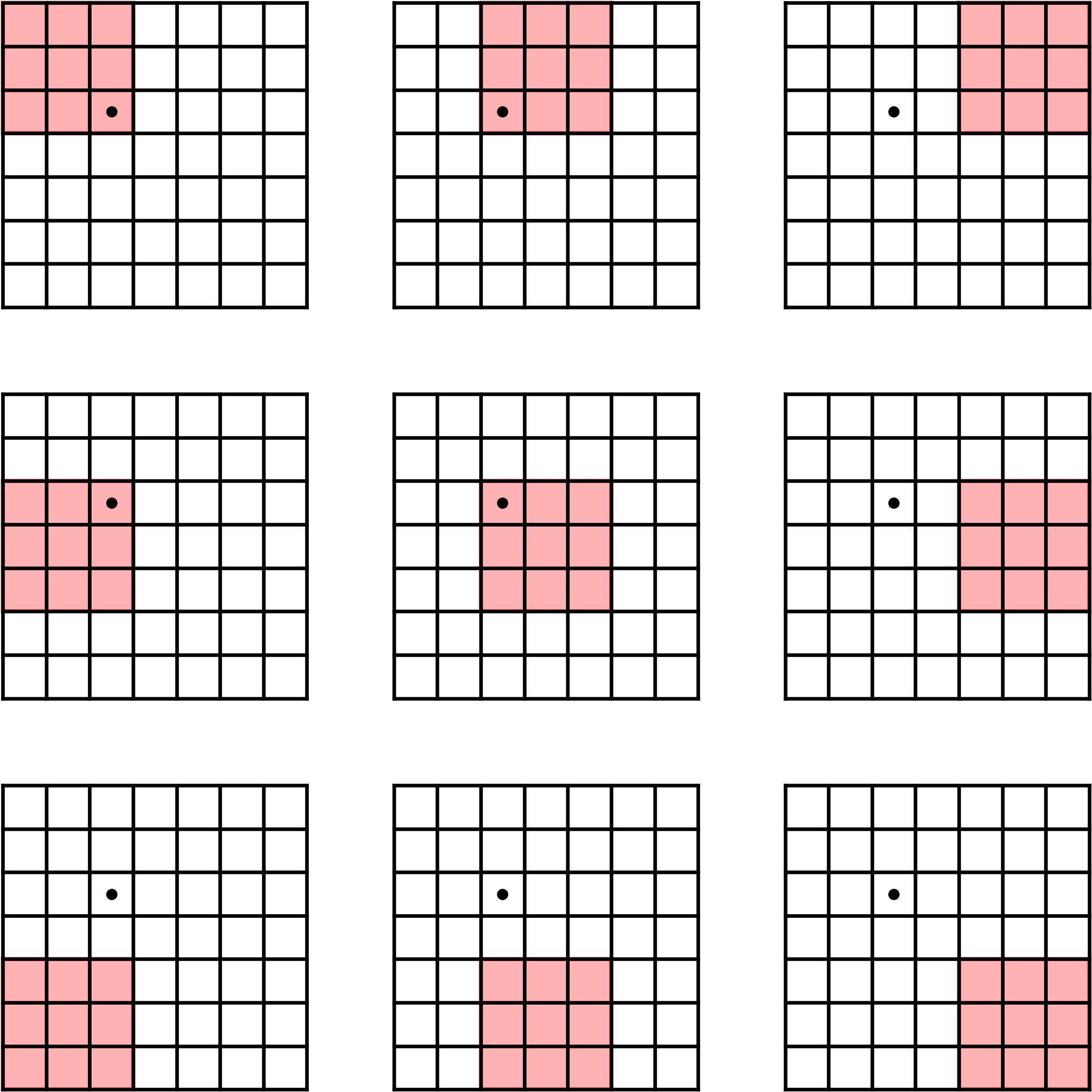}
\caption{Illustration of a 2D max pooling function. Each of the nine grids represent the same input array, but with the kernel (red square) placed in a different location. The dot represents one particular input, which is covered the maximum number of times by the kernel. In this example, the kernel size is $k=3$, the stride size is $s=2$, and the number of strides required to move the kernel to a completely new set of inputs is $c = \mathrm{ceil}(k/s) = 2$ (see Proposition \ref{prp:n_max}). The maximum number of pooling regions that any input can be a part of is $n_{\text{max}} = c^2 = 4$.}
\label{fig:maxpool_grid}
\end{figure}

Next, we show how the value $n_{\text{max}}$ can be computed.
\begin{prp} \label{prp:n_max}
Consider a 2D max pooling function with a dilation of one, and with kernel size $k$ and stride size $s$ in each dimension. Let $n_{\text{max}}$ denote the maximum number of pooling regions that any input can be part of. The value of $n_{\text{max}}$ is given by the following equation:
\begin{equation}
n_{\text{max}} = \mathrm{ceil}(k/s)^2 .
\label{eq:n_max_k_s}
\end{equation}
\end{prp}
\begin{proof}
We start by considering this problem in each of the two dimensions independently. Let $c$ denote the maximum number of times any input can be covered by different placements of the kernel, which also corresponds to the number of the strides required to move the kernel to an entirely different set of inputs (see Fig. \ref{fig:maxpool_grid}). The values $k$, $s$, and $c$ are related by the equation $c s \geq k$. Since $c$ must be an integer, it can be determined with the equation $c = \text{ceil}(k/s)$.


Note that $c$ represents the maximum number of kernel placements in each dimension that can cover a particular input. Since the kernel moves in strides along a 2D grid, $c^2 = \text{ceil}(k/s)^2$ represents the maximum number of kernel placements over both dimensions that can cover a particular input, which is equivalent to $n_{\text{max}}$.

\end{proof}
Note that this result can easily be generalized to the case in which the max pooling function has different kernel sizes and/or stride sizes in each dimension.


\section{Network-wide bounds} \label{sec:network_wide_bounds}

\subsection{Summary of Lipschitz constants and bounds}

We have derived Lipschitz constants and Lipschitz bounds for several functions, most of which describes a single layer of a network. These bounds are summarized in Table \ref{tab:lipschitz_summary}. In this section we will describe how to combine these bounds to determine a network-wide bound.

\newcommand{\specialcell}[2][c]{%
\begin{tabular}[#1]{@{}c@{}}#2\end{tabular}}

\renewcommand{\arraystretch}{1.2}
\begin{table}[ht]
\setlength{\tabcolsep}{4pt}
\centering
\begin{tabular}{ c c c c }
\toprule
\textbf{Function} & \textbf{Global/Local} & \multicolumn{1}{c}{\specialcell{\textbf{Exact/} \\ \textbf{Upper Bound}}} & \textbf{Value} \\
\midrule
affine & global & exact & $\norm{\<A>}$ \\[0pt]
ReLU & global & exact & $1$ \\[0pt]
ReLU & local & exact & $\frac{\relu(\bar{y}) - \relu(y_0)}{\bar{y} - y_0}$ \\[0pt]
affine-ReLU & global & exact & $\norm{\<A>}$ \\
affine-ReLU & local & upper bound & $\norm{\<R> \<A> \<D>}$ \\
max pooling & global & exact & $\sqrt{n_{\text{max}}}$ \\
\bottomrule
\end{tabular}
\vspace{9pt}
\caption{Summary of the Lipschitz constants and bounds derived and discussed in this paper, whether they are global or local measures, and whether they are exact Lipschitz constants or upper bounds. The equation for the local Lipschitz constant of ReLU functions assumes $\mathcal{Y} \neq \{ y_0 \}$.}
\label{tab:lipschitz_summary}
\end{table}
\renewcommand{\arraystretch}{1.0}
\setlength{\tabcolsep}{6pt}

\subsection{Determining the zero output indices of a layer} \label{sec:zero_output_elements_determination}

As we mentioned in Section \ref{sec:input_set}, we consider the input set $\mathcal{X}$ of each layer of a network to be a function of a bound $\epsilon$ on the inputs, and a domain-restriction matrix $\<D>$.
Given the current layer of a network, we now describe how to determine $\<D>$ for the following layer.

For affine-ReLU functions, our goal is to determine which entries of the output set $\mathcal{Z}$ are zero. The zero elements of the vectors $\<z> \in \mathcal{Z}$ can easily be determined from the upper bound vector $\overline{\<y>}$. Since $\overline{\<y>}$ is an upper bound on $\<y> \in \mathcal{Y}$, if $\bar{y}_i \leq 0$ then $y_i \leq 0$ for all $\<y> \in \mathcal{Y}$, which implies $z_i = \relu(y_i) = 0$ for all $\<z> \in \mathcal{Z}$.
Therefore, given an affine-ReLU function, we can form the domain-restriction matrix $\<D>^{\text{next}} \in \mathbb{R}^{m \times m}$ of the next layer as follows:
\begin{align}
\begin{aligned}
d^{\text{next}}_i &= 
\begin{cases}
0, &\bar{y}_i \leq 0 \\
1, &\bar{y}_i > 0
\end{cases} \\
\<D>^{\text{next}} &= \<diag>(d_1^{\text{next}}, \cdots, d_m^{\text{next}}) .
\end{aligned}
\label{eq:di_D_next_layer_aff_relu}
\end{align}

We can also determine the zero output indices of max pooling layers
by noting that if all inputs to a particular pooling region are known to be zero, then the output must be zero.
We can efficiently determine these indices by letting $\<d> \in \mathbb{R}^n$ denote the diagonal elements of the input domain-restriction matrix $\<D> \in \mathbb{R}^{n \times n}$. Each of these elements is a binary value indicating whether the $i^{th}$ element of all inputs equals zero. Therefore, if we plug this vector into the max pooling function,
then each output will equal zero if and only if all inputs in its pooling region are zero, and will equal one otherwise. Therefore, we can use these outputs to form the diagonal elements of the domain-restriction matrix for the next layer:
\begin{align}
\begin{aligned}
\<d>^{\text{next}} &= \<maxpool>(\<d>) \in \mathbb{R}^m \\
\<D>^{\text{next}} &= \<diag>(\<d>^{\text{next}}) \in \mathbb{R}^{m \times m} .
\label{eq:di_D_next_layer_max_pool}
\end{aligned}
\end{align}

Finally, for affine layers, we let the domain-restriction matrix equal the identity matrix:
\begin{align}
\<D>^{\text{next}} = \<I> .
\label{eq:D_next_affine}
\end{align}
Note that affine layers are usually the final layer of a network, so we usually do not have to use this equation.

\subsection{Network-wide bounds} \label{sec:multiple_layers}

\begin{algorithm}
\caption{Steps in our method to compute a bound on the local Lipschitz constant of a feedforward network}
\label{alg:steps}
\begin{algorithmic} 
\State initialize nominal input $\<x>_0$
\State initialize input perturbation size $\epsilon$
\State initialize domain-restriction matrix $\<D>$ as identity matrix
\State initialize network local Lipschitz bound: $L^{\text{net}} \leftarrow 1$
\For{each layer \textbf{in} network}
\If{layer \textbf{is} affine-ReLU}
\State determine $\<y>_0$ using \eqref{eq:y0_z0}: $\<y>_0 \leftarrow \<A> \<x>_0 + \<b>$
\State compute $\overline{\<y>}$ using \eqref{eq:ybar_i}
\State compute $\<R>$ using \eqref{eq:ri_R}
\State compute Lipschitz bound using \eqref{eq:affine_relu_Lipschitz_bound}: $L \leftarrow \norm{\<R> \<A> \<D>}$
\State set nominal input for next layer: $\<x>_0 \leftarrow \<relu>(\<A> \<x>_0 + \<b>)$
\ElsIf{layer \textbf{is} max pooling}
\State compute $n_{\text{max}}$ using \eqref{eq:n_max_k_s}
\State compute Lipschitz constant using \eqref{eq:max_pooling_lipschitz}: $L \leftarrow \sqrt{n_{\text{max}}}$
\State set nominal input for next layer: $\<x>_0 \leftarrow \<maxpool>(\<x>_0)$
\ElsIf{layer \textbf{is} affine}
\State compute Lipschitz constant: $L \leftarrow \norm{\<A>}$
\State set nominal input for next layer: $\<x>_0 \leftarrow \<A> \<x>_0 + \<b>$
\EndIf
\State compute $\<D>$ for next layer using \eqref{eq:di_D_next_layer_aff_relu}, \eqref{eq:di_D_next_layer_max_pool}, or \eqref{eq:D_next_affine}
\State compute $\epsilon$ for next layer: $\epsilon \leftarrow \epsilon L$ (see Proposition \ref{prp:perturbation})
\State update $L^\text{net}$: $L^\text{net} \leftarrow L^\text{net} L$ (see Proposition \ref{prp:composition})
\EndFor
\end{algorithmic}
\end{algorithm}


We now have all of the tools to compute a local Lipschitz bound on a feedforward neural network.
The steps are shown in Algorithm \ref{alg:steps}.
In summary, we start with a nominal input $\<x>_0$ and a bound $\epsilon$ on the set of inputs.
We iterate through each layer of the network, and calculate the Lipschitz constant or bound of the layer.
Then, we determine the nominal input $\<x>_0$, domain-restriction matrix $\<D>$, and input perturbation bound $\epsilon$ for the next layer.
We then update the network Lipschitz bound, and continue iterating through the layers of the network.


\subsection{Relationship between local Lipschitz constants and adversarial bounds} \label{sec:local_lipschitz_adversarial_bounds}

One useful application of local Lipschitz constants is that they can be used to bound adversarial examples. Since the local Lipschitz constant represents how much a network's output can change with respect to changes in the input, it can be used to determine a bound on input perturbations that can change the classification of a classification network. The following proposition describes how input perturbations can be related to adversarial bounds. 

\begin{prp} \label{prp:adversarial}
Consider a feedforward classification neural network $\<f>$ with nominal input $\<x>_0$ and input set $\mathcal{X}$. Assume the input set is norm bounded by $\epsilon$, i.e., $\norm{\<x> - \<x>_0} \leq \epsilon, ~ \forall \<x> \in \mathcal{X}$. Let $\delta \geq 0$ denote the difference between the largest and second-largest values of $\<f>(\<x>_0)$ (i.e., the top two classes). Any $\epsilon$ that satisfies the following equation is a lower bound on the minimum adversarial perturbation (i.e., such a perturbation cannot change the network classification):
\begin{align}
\epsilon L(\<x>_0, \mathcal{X}) < \frac{\delta}{\sqrt{2}} .
\label{eq:eps_L_delta}
\end{align}
\end{prp}
\begin{proof}
Let $\<y> = \<f>(\<x>)$ and $\<y>_0 = \<f>(\<x>_0)$, and let $y_i$ and $y_{0,i}$ denote the $i^{th}$ entries of $\<y>$ and $\<y>_0$, respectively.
Let $a$ and $b$ denote the indices $i$ of the largest and second-largest elements of $\<y>_0$, respectively (i.e., the indices of the first and second classes of the nominal input).
The top-1 classification will change when $y_i$ is greater than or equal to $y_a$ for some $i \neq a$.
Therefore, the smallest value of $\norm{\<y> - \<y>_0}$ for which the classification can change can be derived as follows
\begin{align}
\min_{\substack{i \neq a \\ y_i \geq y_a}} \norm{\<y> - \<y>_0}
&= \min_{\substack{i \neq a \\ y_i = y_a}} \norm{\<y> - \<y>_0} \label{eq:min_y_y0_1st} \\
&= \min_{\substack{i \neq a \\ y_i = y_a}} \sqrt{(y_1 - y_{0,1})^2 + \cdots + (y_m - y_{0,m})^2} \label{eq:min_y_y0_2nd} \\
&= \min_{\substack{i \neq a \\ y_i}} \sqrt{(y_i - y_{0,a})^2 + (y_i - y_{0,i})^2} \label{eq:min_y_y0_3rd} \\
&= \min_{i \neq a} (y_{0,a} - y_{0,i})/\sqrt{2} \label{eq:min_y_y0_4th} \\
&= (y_{0,a} - y_{0,b})/\sqrt{2} \label{eq:min_y_y0_5th} \\
&= \delta/\sqrt{2} \label{eq:min_y_y0_6th} .
\end{align}
The RHS of equation \eqref{eq:min_y_y0_1st} comes from the fact that we can associate any case for which $y_i \geq y_a$ with the case for which $y_i = y_a$, which will have a lower value of $\norm{\<y> - \<y>_0}$. 
Equation \eqref{eq:min_y_y0_3rd} comes from noting that for all scenarios in which $y_i = y_a$, the value $\norm{\<y> - \<y>_0}$ will be the lowest when all entries of $\<y>$ and $\<y>_0$ are equal, except for those corresponding to indices $i$ and $a$.
Equation \eqref{eq:min_y_y0_4th} comes from noting that the expression in the square root of \eqref{eq:min_y_y0_3rd} is minimized when $y_i = (y_{0,a} + y_{0,i})/2$.
Equation \eqref{eq:min_y_y0_5th} comes from noting that the minimum value of $y_{0,a} - y_{0,i}$ occurs when $i=b$.
Equation \eqref{eq:min_y_y0_6th} comes from the definition of $\delta$.

In summary, for any $\<y>$, if $\norm{\<y> - \<y>_0} < \delta/\sqrt{2}$, then the network classification cannot change. Therefore, the network classification will not change if
\begin{equation}
\norm{\<f>(\<x>) - \<f>(\<x>_0)} < \frac{\delta}{\sqrt{2}} .
\label{eq:delta_norm_delta}
\end{equation}

Next, from Proposition \ref{prp:perturbation}, if $\norm{\<x> - \<x>_0} \leq \epsilon$ for all $\<x> \in \mathcal{X}$, then $\norm{\<f>(\<x>) - \<f>(\<x>_0)} \leq \epsilon L(\<x>_0, \mathcal{X})$ for all $\<x> \in \mathcal{X}$.
Therefore, if $\epsilon L(\<x>_0, \mathcal{X}) < \delta/\sqrt{2}$ then $\norm{\<f>(\<x>) - \<f>(\<x>_0)} < \delta/\sqrt{2}$, which using \eqref{eq:delta_norm_delta} tells us that the classification cannot change.

\end{proof}

In practice, we will use Proposition \ref{prp:adversarial} to find the largest $\epsilon$ which is a lower bound on the minimum adversarial perturbation (i.e., that satisfies \eqref{eq:eps_L_delta}). Note that $L(\<x>_0, \mathcal{X})$ is a non-decreasing function of $\epsilon$ (a consequence of Proposition \ref{prp:superset}), so $\epsilon L(\<x>_0, \mathcal{X})$ is also non-decreasing in $\epsilon$. This means that we can determine the largest possible $\epsilon$ by increasing $\epsilon$ until \eqref{eq:eps_L_delta} is no longer satisfied. Section \ref{sec:adversarial_application} shows simulations in which we apply this technique.

\section{Computational techniques} \label{sec:computational_techniques}

We will now discuss two computational insights that make it possible to apply our method to large layers and networks.


Our first computational insight concerns efficiently calculating $\overline{\<y>}$, which is the upper bound of the set $\mathcal{Y}$. The equation for $\overline{\<y>}$ is shown in \eqref{eq:ybar_i}, and requires determining $\<a>_i^T$, the $i^{th}$ row of the $\<A>$ matrix. For large convolutional layers, the $\<A>$ matrices are usually too large to store in random-access memory. So instead of determining the entire $\<A>$ matrix, we can obtain the $i^{th}$ row of $\<A>$ by noting that $\<a>_i^T = \<A>^T \<e>_i$ where $\<e>_i \in \mathbb{R}^m$ is the $i^{th}$ standard basis vector. To perform the $\<A>^T$ transformation we can use a transposed convolution function based on the original convolution function (making sure to reshape $\<e>_i$ into the appropriate input size).
Furthermore, to reduce computation time, we can use a batch of standard basis vectors in the transposed convolution function to obtain multiple rows of $\<A>$.


Our second computational insight concerns efficiently computing $\norm{\<R> \<A> \<D>}$, which is the affine-ReLU local Lipschitz constant bound from Theorem \ref{thm:aff_relu}. The matrix $\<R> \<A> \<D>$ is usually too large to be stored in memory, so we use a power iteration to compute it.
Note that the largest singular value of a matrix $\<M>$ is the square root of the largest eigenvalue of $\<M>^T \<M>$. So, we can find the spectral norm of $\<M>$ by applying a power iteration to the operator $\<M>^T \<M>$. In our case, our matrix is $\<M> = \<R> \<A> \<D>$, so we have $\<M>^T \<M> = \<D>^T \<A>^T \<R>^T \<R> \<A> \<D> = \<D> \<A>^T \<R>^2 \<A> \<D>$. For convolutional layers, we can perform the $\<A>$ transformation using the convolution function (with zero bias) and the $\<A>^T$ transformation using transposed convolution (with zero bias). Furthermore, since both $\<R>^2$ and $\<D>$ are diagonal matrices, we can apply these transformations using elementwise vector multiplication.
\section{Simulations} \label{sec:simulations}

\subsection{Local Lipschitz constants for various networks}

\begin{figure}[ht]

\newcommand\hgt{0.45}
\newcommand{\spc}{\hspace{4.0pt}}
\newcommand{\wid}{.485\textwidth}

\footnotesize

\centering
\begin{tikzpicture}[every text node part/.style={align=center}]

\colorlet{color1}{black!15}
\colorlet{color2}{black!5}

\fill[color1] (0,0) rectangle (\wid,-\hgt);
\fill[color2] (0,-\hgt) rectangle (\wid,-2*\hgt);
\fill[color1] (0,-2*\hgt) rectangle (\wid,-3*\hgt);
\fill[color2] (0,-3*\hgt) rectangle (\wid,-4*\hgt);

\node (input) at (.5*\wid,-.5*\hgt)
	{\textbf{CIFAR-10 Net Architecture}};
\node[anchor=west] (input) at (0,-1.5*\hgt)
	{C3-32 \spc C3-32 \spc MP-2 \spc  D  \spc C3-64 \spc  C3-64  \spc MP-2 \spc  D  \spc FC-512  \spc D  \spc FC-10};

\node (input) at (.5*\wid,-2.5*\hgt)
	{\textbf{MNIST Net Architecture}};
\node[anchor=west] (input) at (0,-3.5*\hgt)
	{C5-6 \spc MP-2 \spc C5-16 \spc MP-2 \spc FC-120 \spc FC-84 \spc FC-10};

\end{tikzpicture}

\caption{Architectures of the networks we constructed for this paper (in sequence left-to-right). ``C$\alpha$-$\beta$'' denotes a convolution layer with kernel size $\alpha$ and $\beta$ output channels, ``MP-$\alpha$'' denotes a max pooling layer with kernel size $\alpha$, ``FC-$\alpha$'' denotes a fully-connected layer with $\alpha$ output features, and ``D'' denotes dropout layers. All convolution layers are followed by a ReLU and have a stride of 1. All fully-connected layers are followed by a ReLU unless it is the last layer.}
\label{fig:network_architectures}
\end{figure}

\begin{figure*}[ht]
\centering
\includegraphics[width=.32\textwidth]{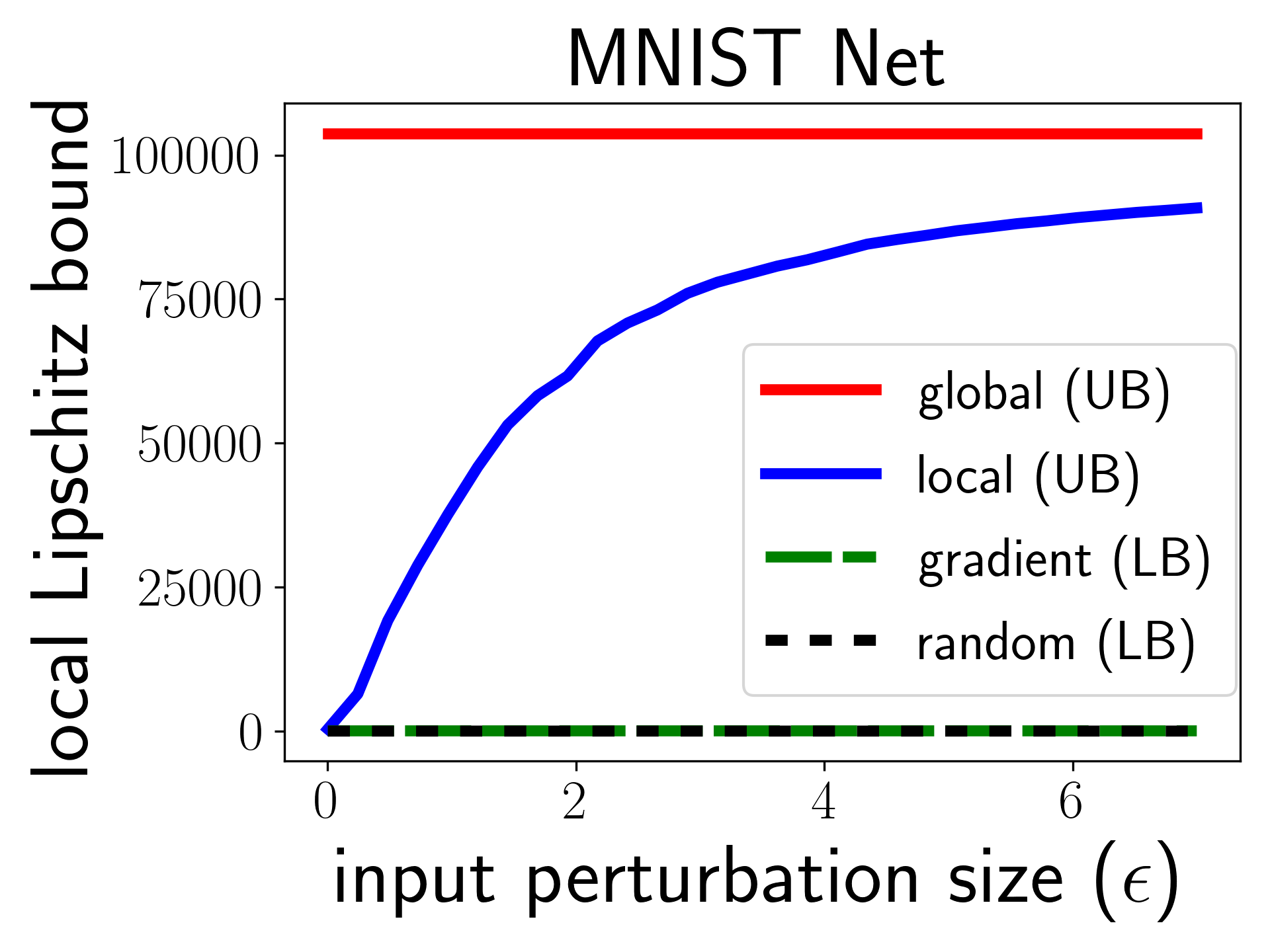}
\includegraphics[width=.32\textwidth]{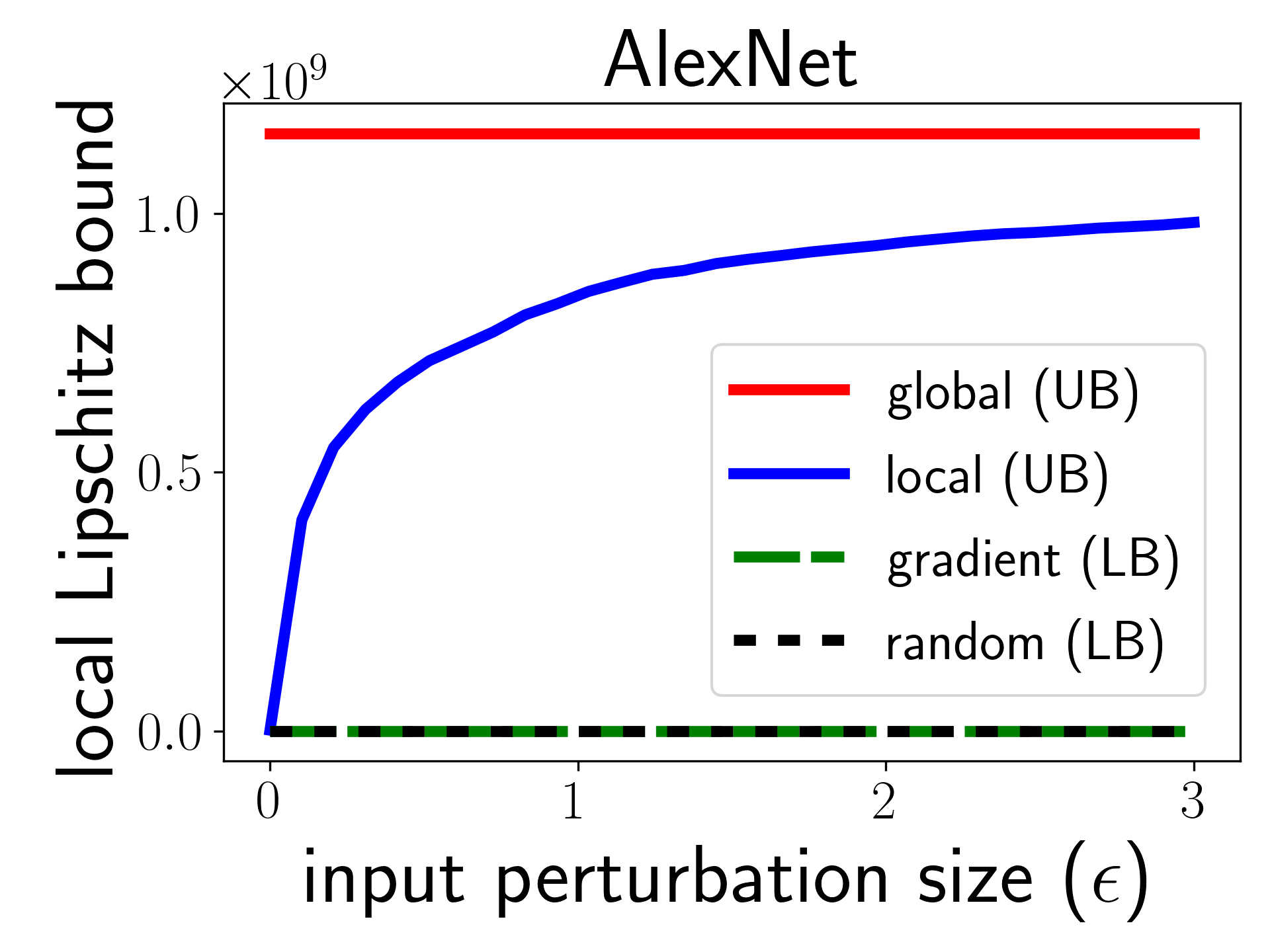}
\includegraphics[width=.32\textwidth]{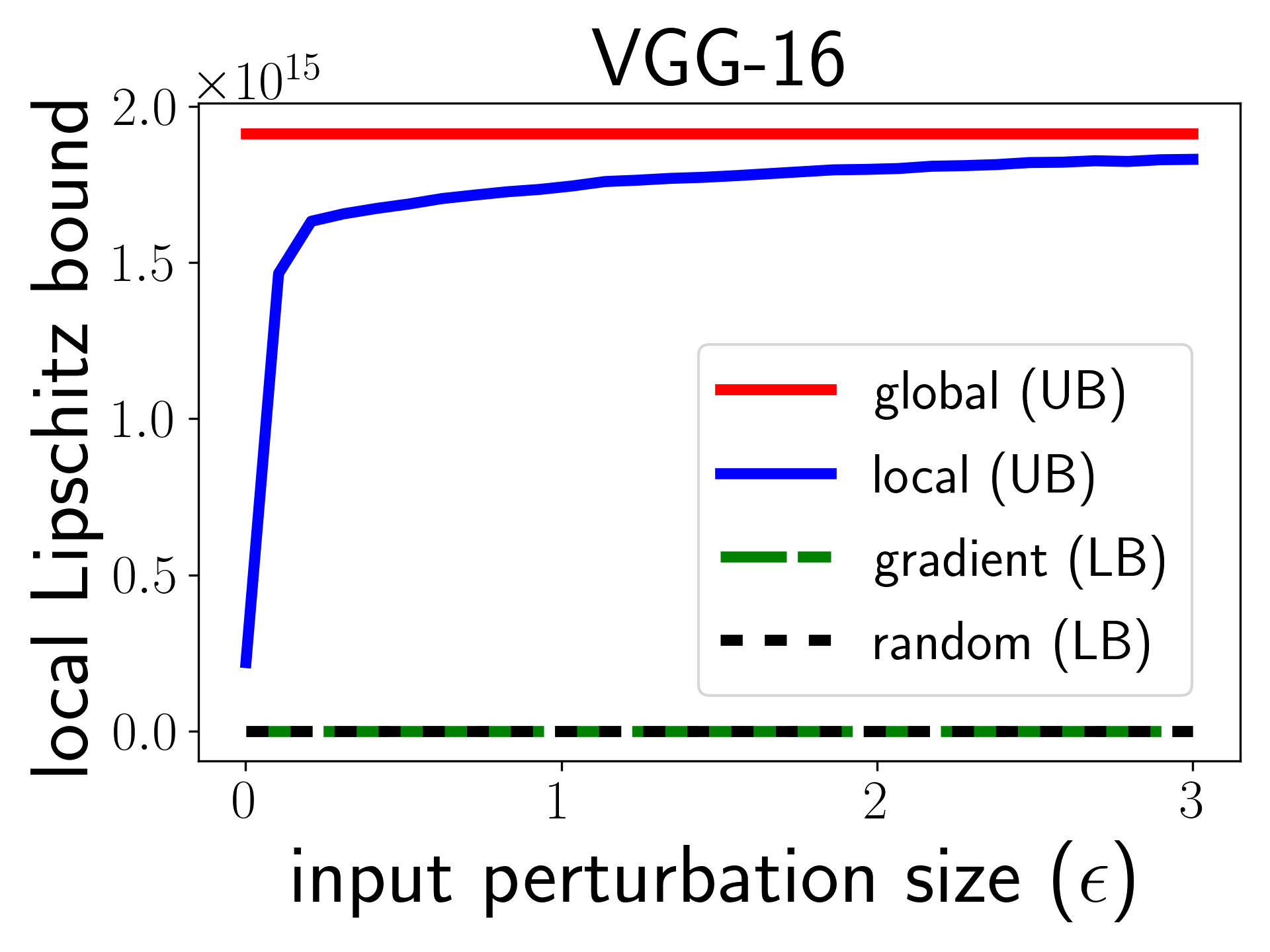}
\caption{Upper bounds (UB) and lower bounds (LB) on the local Lipschitz constants of the MNIST Net, AlexNet, and VGG-16 networks for various perturbation sizes (the plot for CIFAR-10 Net is not shown, but has similar trends). These results are computed with respect to the nominal input images in Fig. \ref{fig:nominal_inputs}. The term ``global'' represents to global bounds computed using Theorem \ref{thm:aff_relu} for affine-ReLU functions, ``local'' refers to local bounds computed using $\norm{\<A>}$ for affine-ReLU functions, ``gradient'' refers to lower bounds determined using gradient ascent, and ``random'' refers to lower bounds determined by sampling random input perturbations.}
\label{fig:network_lipschitz}
\end{figure*}

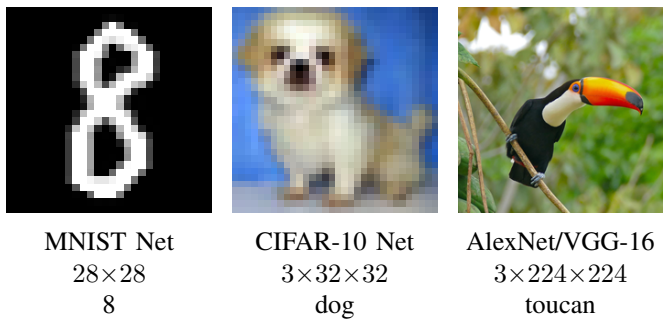
\begin{figure}[ht]

\begin{tikzpicture}[every text node part/.style={align=center}]
\node[inner sep=0pt] (mnist) at (0,0)
    {\includegraphics[width=.15\textwidth]{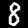}};
\node[inner sep=0pt] (cifar10) at (3,0)
    {\includegraphics[width=.15\textwidth]{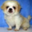}};
\node[inner sep=0pt] (imagenet) at (6,0)
    {\includegraphics[width=.15\textwidth]{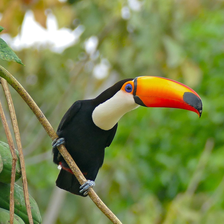}};

\node[anchor=south,below] (lipschitz) at (0,-1.5) {MNIST Net \\ $28{\times}28$ \\ 8};
\node[anchor=south,below] (lipschitz) at (3,-1.5) {CIFAR-10 Net \\ $3{\times}32{\times}32$ \\ dog};
\node[anchor=south,below] (lipschitz) at (6,-1.5) {AlexNet/VGG-16 \\ $3{\times}224{\times}224$ \\ toucan};

\end{tikzpicture}

\caption{Images, networks, image sizes, and true class names of the nominal input images used in our simulations.}

\label{fig:nominal_inputs}
\end{figure}

We compare four different networks in this paper: a seven-layer network trained on MNIST (which we refer to as ``MNIST Net''), an eight-layer network trained on CIFAR-10 (which we refer to as ``CIFAR-10 Net''), AlexNet \cite{Krizhevsky} (11-layers, trained on ImageNet), and VGG-16 (21 layers including 16 affine-ReLU layers, trained on ImageNet) \cite{Simonyan}. The architectures of MNIST Net and CIFAR-10 Net are shown in Fig. \ref{fig:network_architectures}. We constructed MNIST Net ourselves and trained it to 99\% top-1 test accuracy in 100 epochs. We also constructed CIFAR-10 Net ourselves, and trained it to 84\% top-1 test accuracy in 500 epochs. We used the trained versions of AlexNet and VGG-16 from Pytorch's Torchvision package. Note these networks both have an ``adaptive average pooling'' layer which has no effect when the network inputs are the default size of $3{\times}224{\times}224$.

All simulations were performed using Pytorch, and were run on an Nvidia GTX 1080 Ti card.
In our simulations we used the nominal input images shown in Fig. \ref{fig:nominal_inputs}, which all classify correctly.
In each simulation we determined global and local upper bounds, as well as lower bounds computed using both gradient ascent and random sampling methods.
Fig. \ref{fig:network_lipschitz} shows the full-network local Lipschitz bounds, and Table \ref{tab:computation_times} shows the computation times.


The results show that our Lipschitz bounds increase with the size of the perturbation $\epsilon$, and approach the global bound for large $\epsilon$. For small perturbations, the bound is significantly lower than the global bound.



\begin{table}[ht]
\centering
\begin{tabular}{ l c c c c }
\toprule
& \textbf{MNIST Net} & \textbf{CIFAR-10 Net} & \textbf{AlexNet} & \textbf{VGG-16} \\
\midrule
time & .1 sec & 1 sec & 16 sec & 52 min \\ 
\bottomrule
\end{tabular}
\vspace{9pt}
\caption{Times to compute the local Lipschitz constant upper bound for one input perturbation of size $\epsilon$ for various networks (using Algorithm \ref{alg:steps}), based on the nominal input images in Fig. \ref{fig:nominal_inputs}. All computations were performed on a desktop computer using Pytorch and an Nvidia GTX 1080 Ti card.}
\label{tab:computation_times}
\end{table}

\subsection{Bounds on adversarial examples} \label{sec:adversarial_application}

Next, we apply our local Lipschitz bounds to determine lower bounds on adversarial perturbations, as described in Section \ref{sec:local_lipschitz_adversarial_bounds}. We calculated these bounds with respect to the same networks and nominal input images in Fig. \ref{fig:nominal_inputs} in Section \ref{sec:simulations}. The results are shown in Table \ref{tab:classification_bounds}.

\begin{table}[ht]
\centering
\setlength{\tabcolsep}{5pt}
\renewcommand{\arraystretch}{1.2}
\begin{tabular}{ l l l l l l }
\toprule
& \multicolumn{1}{c}{} & \multicolumn{1}{c}{\specialcell{\textbf{MNIST} \\ \textbf{Net}}} & \multicolumn{1}{c}{\specialcell{\textbf{CIFAR-10} \\ \textbf{Net}}} & \multicolumn{1}{c}{\textbf{AlexNet}} & \multicolumn{1}{c}{\textbf{VGG-16}} \\
\midrule
\parbox[t]{4mm}{\multirow{2}{*}{\rotatebox[origin=cB]{90}{\parbox{.6cm}{\centering \textit{upper} \\ \textit{bound}}}}} &
FGSM &
$1.8 \cdot 10^{1}$ &
$6.7 \cdot 10^{0}$ &
$7.8 \cdot 10^{0}$ & 
$4.7 \cdot 10^{1}$ \\
& gradient &
$4.0 \cdot 10^{0}$ &
$2.6 \cdot 10^{0}$ &
$4.6 \cdot 10^{0}$ & 
$2.7 \cdot 10^{0}$ \\
\midrule
\parbox[t]{4mm}{\multirow{2}{*}{\rotatebox[origin=cB]{90}{\parbox{.6cm}{\vspace{-.5mm}\centering \textit{lower} \\ \textit{bound}}}}} &
local &
$ 5.2 \cdot 10^{-2}$ &
$ 4.2 \cdot 10^{-3}$ &
$ 1.0 \cdot 10^{-5}$ &
$ 1.8 \cdot 10^{-8}$ \\
& global &
$5.1 \cdot 10^{-4}$ &
$3.9 \cdot 10^{-4}$ &
$8.6 \cdot 10^{-9}$ & 
$7.4 \cdot 10^{-15}$ \\ 
\bottomrule
\end{tabular}
\vspace{9pt}
\caption{Bounds on the minimum Euclidean perturbation required to change the top classification of a classification network, based on the nominal input images shown in Fig. \ref{fig:nominal_inputs}. The ``global'' row refers to bounds computed using the global Lipschitz constant, ``local'' refers to bounds computed using the local Lipschitz constant, ``gradient'' refers to lower bounds determined by finding adversarial examples using gradient ascent, and ``FGSM'' refers to lower bounds determined using the Fast Gradient Sign Method \cite{Goodfellow2014}.}
\label{tab:classification_bounds}
\end{table}
\setlength{\tabcolsep}{6pt}
\renewcommand{\arraystretch}{1.0}

For each network, our method provides an orders of magnitude improvement over the global bound. We are unaware of any other method that can improve upon the global bound for networks such as AlexNet and VGG-16, so our bounds represent a significant improvement in certifying adversarial bounds to Euclidean perturbations.

Also shown in Table \ref{tab:classification_bounds} are upper bounds computed using gradient methods. Having both lower and upper bounds allows us to identify the range in which the true minimum perturbation resides.

\subsection{Comparison with other methods}

As mentioned in Section \ref{sec:intro}, there are several methods which provide Lipschitz estimates or bounds, but many only work for small networks. We did not apply the methods in \cite{Jordan,Latorre} as they have only been shown to be applicable to networks smaller than the smallest network we considered (MNIST Net).
We were able to implement the method in \cite{Fazlyab} to MNIST Net, but it ran out of memory for the larger networks. As this method is not designed to incorporate max pooling functions, we used it estimate the global Lipschitz constant of each affine-ReLU sequence, and combined the results with the max pooling global bound, which resulted in an estimate of $67{\times}10^3$. We also were able to apply the estimation method in \cite{Scaman}. We note that \cite{Scaman} presents two bounds: AutoLip which is equivalent to the global bound, and SeqLip. SeqLip produced estimates of $72{\times}10^3$, $7{\times}10^3$, and $174{\times}10^6$ for MNIST Net, CIFAR-10 Net, and AlexNet, respectively, and took longer than 48 hours to produce an estimate for VGG-16 so we aborted the operation.

\section{Conclusion} \label{sec:conclusion}

We have presented a method to determine guaranteed upper bounds on the local Lipschitz constant of neural networks with ReLU activations.
Our approach is based on determining Lipschitz constants and bounds of ReLU, affine-ReLU and max pooling functions. We then showed how we can calculate these Lipschitz constants/bounds in a sequential fashion for each layer of a feedforward network, which allows us to compute a network-wide bound.

We calculated our bounds for small MNIST and CIFAR-10 networks, as well as large networks such as AlexNet and VGG-16. The results show that our bounds are especially tight for small perturbations. We then showed how we can use our method to determine lower bounds on Euclidean adversarial perturbations. To the best of our knowledge, our method produces the tightest known bounds for larger networks.

Potential future work includes reducing the computation time of our method, as well as further mathematical analysis to obtain even tighter bounds. Note that for larger layers, the main computational bottleneck comes from computing $\overline{\<y>}$, which requires evaluating each row of the $\<A>$ matrix.

Finally, there are several ways in which this work could be extended. For example, we could consider activation functions other than ReLU, and we could consider other types of layers.
Additionally, as we only considered the 2-norm, we could generalize our results to other norms.
Note that many of our results hold for general matrix norms, and we believe many of the other results could be generalized without too much trouble.





\bibliographystyle{IEEEtran}
\bibliography{root.bib}

\end{document}